\documentclass[conference]{IEEEtran}
\usepackage{times}
\IEEEoverridecommandlockouts                              

\usepackage[numbers]{natbib}
\usepackage{multicol}
\usepackage[bookmarks=true]{hyperref}

\usepackage{amssymb}        
\usepackage{amsthm}
\usepackage{subcaption}

\usepackage[font=small]{caption}
\usepackage{xspace}
\usepackage{tabularx}
\usepackage{graphicx}       
\newtheorem{theorem}{Theorem}
\newtheorem{lemma}[theorem]{Lemma}
\newtheorem{proposition}[theorem]{Proposition}
\newtheorem{corollary}[theorem]{Corollary}
\theoremstyle{definition}

\theoremstyle{remark}

\newtheorem{problem}{Problem}
\usepackage{marginnote}     
\usepackage{marvosym}       
\usepackage{filecontents}
\usepackage{xcolor}
\usepackage{wrapfig}
\usepackage{overpic}
\usepackage{comment}

\newcommand{\rtp}{\textsc{RTP}\xspace}
\newcommand{\rta}{\textsc{RTA}\xspace}
\newcommand{\mpp}{\textsc{MRPP}\xspace}
\newcommand{\ecbs}{\textsc{ECBS}\xspace}

\newcommand{\ddm}{\textsc{DDM}\xspace}
\newcommand{\rth}{\textsc{RTH}\xspace}
\newcommand{\rtm}{\textsc{RTLM}\xspace}

\newcommand{\rthlba}{\textsc{RTH-LBA}\xspace}
\newcommand{\rthip}{\textsc{RTH-IP}\xspace}

\def\gtl{\mathcal G_{tl}\xspace}
\def\gbr{\mathcal G_{br}\xspace}

\newif\ifdraft
\draftfalse
\drafttrue

\newif\ifarxiv
\arxivfalse
\arxivtrue

\ifdraft

\usepackage{xcolor}
\usepackage{xargs} 
\usepackage[textsize=footnotesize]{todonotes}
\newcommandx{\sh}[2][1=]{\todo[linecolor=blue,
			backgroundcolor=blue!10,bordercolor=blue,#1]{Han: #2}}
\newcommandx{\tg}[2][1=]{\todo[linecolor=orange,
			backgroundcolor=orange!10,bordercolor=orange,#1]{Greaten: #2}}
\newcommandx{\jy}[2][1=]{\todo[linecolor=green,
			backgroundcolor=green!10,bordercolor=green,#1]{JJ: #2}}
\else
\newcommand{\sh}[1]{{}}
\newcommand{\tg}[1]{{}}
\newcommand{\jy}[1]{{}}
\fi

\makeatletter
\newif\if@restonecol
\makeatother

\usepackage[linesnumbered,ruled,vlined]{algorithm2e}
\usepackage{algpseudocode}
\usepackage{amsmath}

\begin{document}


\title{{Sub-1.5 Time-Optimal Multi-Robot Path Planning on Grids in Polynomial Time}}

\author{Teng Guo \qquad Jingjin Yu
\thanks{
The authors are with the Department of Computer Science, 
Rutgers, the State University of New Jersey, Piscataway, NJ, 
USA. E-Mails: \{{\tt teng.guo, jingjin.yu}\}@rutgers.edu. 
}%
}

\maketitle

\begin{abstract}
It is well-known that graph-based multi-robot path planning (\mpp) is NP-hard to optimally solve. 
In this work, we propose the first low polynomial-time algorithm for  \mpp achieving 1--1.5 asymptotic optimality guarantees on solution makespan (i.e., the time it takes to complete a reconfiguration of the robots) for random instances under very high robot density, with high probability. The dual guarantee on computational efficiency and solution optimality suggests our proposed general method is promising in significantly scaling up multi-robot applications for logistics, e.g., at large robotic warehouses.

Specifically, on an $m_1\times m_2$ gird, $m_1 \ge m_2$, our RTH (Rubik Table with Highways) algorithm computes solutions for routing up to $\frac{m_1m_2}{3}$ robots with uniformly randomly distributed start and goal configurations with a makespan of $m_1 + 2m_2 + o(m_1)$, with high probability. Because the minimum makespan for such instances is $m_1 + m_2 - o(m_1)$, also with high probability, RTH guarantees $\frac{m_1+2m_2}{m_1+m_2}$ optimality as $m_1 \to \infty$ for random instances with up to $\frac{1}{3}$ robot density, with high probability. $\frac{m_1+2m_2}{m_1+m_2} \in (1, 1.5]$. 
Alongside this key result, we also establish a series of related results supporting even higher robot densities and environments with regularly distributed obstacles, which directly map to real-world parcel sorting scenarios.
Building on the baseline methods with provable guarantees, we have developed effective, principled heuristics that further improve the computed optimality of the \rth algorithms.
In extensive numerical evaluations, \rth and its variants demonstrate exceptional scalability as compared with methods including ECBS and DDM, scaling to over $450 \times 300$ grids with $45,000$ robots, and consistently achieves makespan around $1.5$ optimal or better, as predicted by our theoretical analysis. 
\end{abstract}

\IEEEpeerreviewmaketitle
\section{Introduction}
We examine multi-robot path planning (\mpp, also known as multi-agent path finding or MAPF \cite{stern2019multi}) on two-dimensional grids, with potentially regularly distributed obstacles (see Fig.~\ref{fig:jd_center}). 
The main objective of \mpp is to find a set of collision-free paths for routing many robots from a start configuration to a goal configuration.  
In practice, solution optimality is also of key importance; yet optimally solving \mpp is generally NP-hard~\cite{surynek2010optimization,yu2013structure}, even in planar \cite{yu2015intractability} and grid settings~\cite{demaine2019coordinated}. 
\mpp algorithms find many important large-scale applications, including, e.g., in warehouse automation for general order fulfillment \cite{wurman2008coordinating}, grocery order fulfillment \cite{mason2019developing}, and parcel sorting \cite{wan2018lifelong}. Other application scenarios include formation reconfiguration~\cite{PodSuk04}, agriculture~\cite{cheein2013agricultural}, object 
transportation~\cite{RusDonJen95}, swarm robotics \cite{preiss2017crazyswarm,honig2018trajectory}, to list a few.
 
 \begin{figure}[htbp]
        \centering
        \vspace{1mm}
        \begin{overpic}               
        [width=1\linewidth]{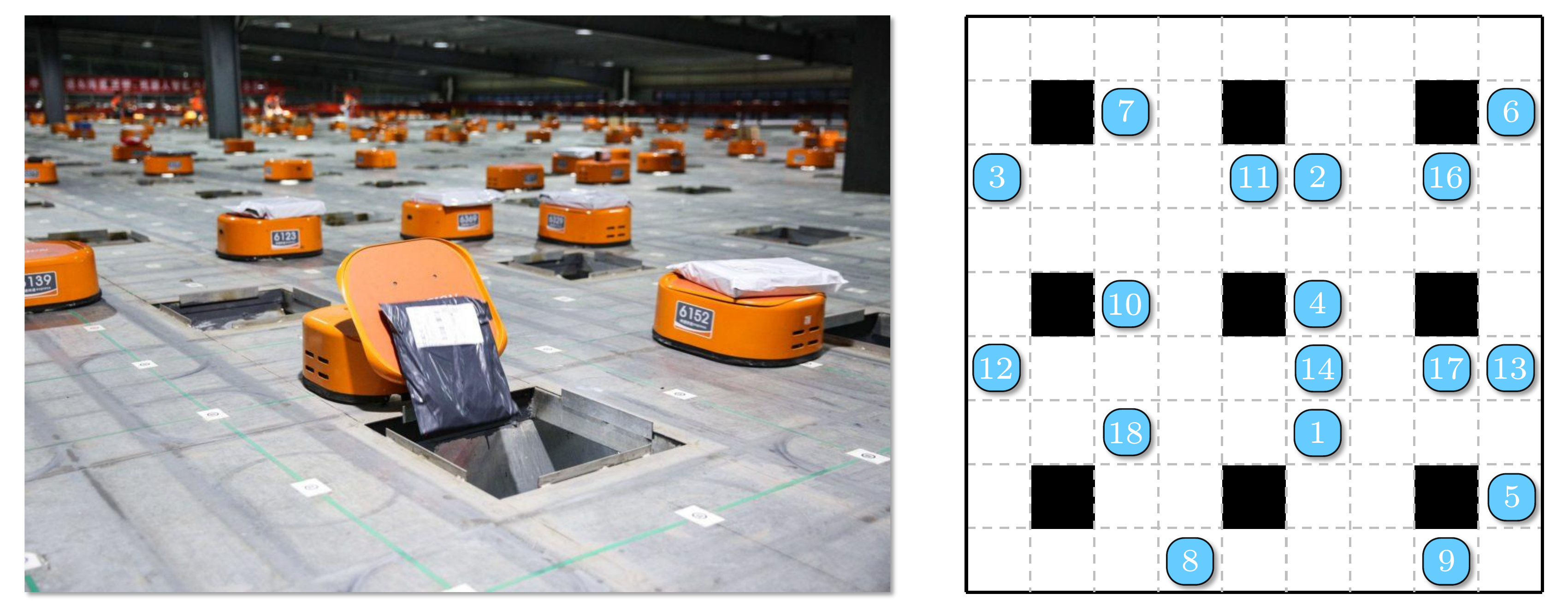}
             \footnotesize
             \put(28.5, -3) {(a)}
             \put(78.5, -3) {(b)}
        \end{overpic}
        \vspace{-2mm}
        \caption{ (a) Real-world parcel sorting system (by JD.com) using many robots on a large grid-like environment with holes for dropping parcels; (b) A snapshot of a similar \mpp instance we can solve in polynomial-time with provable optimality guarantees. In practice, our algorithms scale to maps of size {$450 \times 300$}, supporting over $50$K robots, and achieves $1.x$-optimality (see, e.g.,  Fig.~\ref{fig: sorting_random}). 
        } 
        \label{fig:jd_center}
        \vspace{-5mm}
    \end{figure}

Motivated by applications including grocery fulfillment and parcel sorting, we focus on \mpp in which the underlying graph is an $m_1\times m_2$ grid, $m_1\ge m_2$, with extremely high robot density. Whereas recent studies \cite{yu2018constant,demaine2019coordinated} have shown that such problems can be solved in polynomial time with $O(1)$ optimality guarantees, the constant factor associated with the guarantee is generally prohibitively high ($\gg 1$) for these methods to be practical.
In this research, we break this barrier by showing that, we can achieve $(1+\delta)$-makespan optimality for \mpp on large grids in polynomial-time in which $\delta \in (0, 0.5 + \varepsilon]$, $\varepsilon \to 0$ as $m_1 \to \infty$.
Through the judicious application of a novel global object rearrangement method called Rubik Tables \cite{szegedy2020rearrangement} together with many algorithmic techniques, and combined with careful analysis, we establish that in polynomial time:
\begin{itemize}
    \item For $m_1m_2$ robots, i.e., at maximum robot density, RTM (Rubik Table for \mpp) computes a solution for an arbitrary \mpp instance under a makespan of $7m_1 + 14m_2$;
    \item For $\frac{m_1m_2}{3}$ robots and uniformly randomly distributed start/goal configurations, RTH (Rubik Tables with Highways) computes a solution with a makespan of $m_1 + 2m_2 + o(m_1)$, with high probability. In contrast, such an instance has a minimum makespan of $m_1 + m_2 - o(m_1)$ with high probability. This implies that, as $m_1 \to \infty$, an optimality guarantee of $\frac{m_1 + 2m_2}{m_1+m_2} \in (1, 1.5]$ is achieved, with high probability; 
    \item For $\frac{m_1m_2}{3}$ robots, for an arbitrary (i.e., not necessarily random) instance, a solution can be computed with a makespan of $3m_1+4m_2+o(m_1)$ using RTH; 
    \item For $\frac{m_1m_2}{2}$ robots, the same $\frac{m_1 + 2m_2}{m_1+m_2}$ optimality guarantee can be achieved with a slightly larger overhead using \rtm (Rubik Tables with Line Merge); 
    \item The same $\frac{m_1 + 2m_2}{m_1+m_2}$ optimality guarantee may be achieved on grids with up to $\frac{m_1m_2}{9}$ regularly distributed obstacles together with $\frac{2m_1m_2}{9}$ robots using RTH (e.g., Fig.~\ref{fig:jd_center}(b)). 
\end{itemize}

Moreover, we have developed effective and principled heuristics to work together with \rth that further reduce the computed makespan by a large margin, i.e., for $\frac{m_1m_2}{3}$ robots, a makespan smaller than $m_1 + 2m_2$ can often be achieved. Demonstrated through extensive numerical evaluations, our methods are highly scalable, capable of solving instances with tens of thousands of robots in dense settings under two minutes. Simultaneously, the solution optimality approaches the 1--1.5 range as predicted theoretically. This level of scalability far exceeds what was possible. 
With the sub-1.5 optimality guarantee, our approach unveils a promising direction toward the development of practical, provably optimal multi-robot routing algorithms that runs in low polynomial time. 

\textbf{Related work.} 
Literature on multi-robot path and motion planning \cite{hopcroft1984complexity,ErdLoz86} is expansive; here, we mainly focus on graph-theoretic (i.e., the state space is discrete) studies \cite{yu2016optimal,stern2019multi}. As such, in this paper, \mpp refers explicitly to graph-based multi-robot path planning. 
Whereas the feasibility question has long been positively answered for \mpp  \cite{KorMilSpi84}, the same cannot be said when it comes to securing optimal solutions, as computing time- or distance-optimal solutions are shown to be NP-hard in many settings, including for general graphs \cite{goldreich2011finding,surynek2010optimization,yu2013structure}, planar graphs \cite{yu2015intractability,banfi2017intractability}, and even regular grids \cite{demaine2019coordinated}, similar to the setting addressed in this study.  

Nevertheless, given its high utility, especially in e-commerce applications \cite{wurman2008coordinating,mason2019developing,wan2018lifelong} that are expected to grow significantly \cite{dekhne2019automation,covid-auto}, many algorithmic solutions have been proposed for optimally solving \mpp. 
Among these, combinatorial-search based solvers \cite{lam2019branch} have been demonstrated to be fairly effective.  \mpp solvers may be classified as being optimal or suboptimal. 
Reduction-based optimal solvers solve the problem through reducing the \mpp problem to other problem, e.g., SAT~\cite{surynek2012towards}, answer set programming~\cite{erdem2013general}, integer linear programming (ILP)~\cite{yu2016optimal}.
Search-based optimal \mpp solvers include EPEA* \cite{goldenberg2014enhanced}, ICTS \cite{sharon2013increasing}, CBS \cite{sharon2015conflict}, M* \cite{wagner2015subdimensional}, and many others. 
Due to the inherent intractability of optimal \mpp, optimal solvers usually exhibit limited scalability, leading to considerable interests in suboptimal solvers.
Unbounded solvers like push-and-swap~\cite{luna2011push}, push-and-rotate~\cite{de2014push}, windowed hierarchical cooperative A${}^*$~\cite{silver2005cooperative}, all return feasible solutions very quickly, but at the cost of solution quality.
Balancing the running-time and optimality is one of the most attractive topics in the study of \mpp/MAPF. 
Some algorithms emphasize the scalability without sacrificing as much optimality, e.g., ECBS~\cite{barer2014suboptimal}, DDM \cite{han2020ddm}, EECBS \cite{li2021eecbs}, PIBT \cite{Okumura2019PriorityIW}, PBS \cite{ma2019searching}. There are also learning-based solvers ~\cite{damani2021primal,sartoretti2019primal} that scales well in sparse environments. Effective orthogonal heuristics have also been proposed \cite{9561899}. 
Recently, $O(1)-$approximate or constant factor time-optimal algorithms have been proposed, e.g.  \cite{yu2018constant,demaine2019coordinated}, that tackle highly dense instances. However, these algorithms only achieve low-polynomial time guarantee at the expense of very large constant factors, rendering them theoretically interesting but impractical. 

In contrast, with high probability, our methods run in low polynomial time with provable 1--1.5 asymptotic optimality. To our knowledge, this paper presents the first \mpp algorithms to simultaneously guarantee polynomial running time and $1.x$ solution optimality.

\textbf{Organization.} The rest of the paper is organized as follows.
In Sec.~\ref{sec:pre}, we provide a formal definition of graph-based \mpp, and introduce the Rubik Table problem and the associated algorithm.
RTM, a basic adaptation of the Rubik Table results for \mpp at maximum robot density which ensures a makespan upper bound of $7m_1 + 14m_2$, is described in Sec.~\ref{sec:1:1}. An accompanying lower bound of $m_1 + m_2 - o(m_1)$ for random \mpp instances is also established. 
In Sec.~\ref{sec:1:2} we introduce RTH for one third robot density achieving a makespan of $m_1 + 2m_2 + o(m_1)$. Obstacle support is also discussed. 
In Sec.~\ref{sec:1:3}, we show how one-half robot density may be supported with optimality guarantees similar to that of RTH. 
We thoroughly evaluate the performance of our methods in Sec.~\ref{sec:eval} and conclude with Sec.~\ref{sec:conclusion}.
Given the amount of material included in the work, to provide a concentrated discussion, we refer the readers to \cite{guo2022sub} for proofs to theorems. 

\section{Preliminaries}\label{sec:pre}
\subsection{Multi-Robot Path Planning on Graphs}
Graph-based multi-robot path planning (\mpp) seeks collision-free paths 
that efficiently route robots. 
Consider an undirected graph $\mathcal G(V, E)$ and $n$ robots 
with start configuration $S = \{s_1, \dots, s_n\} \subseteq V$ and goal configuration 
$G = \{g_1, \dots, g_n\} \subseteq V$. Each robot has start and goal vertices $s_i$, $g_i$.
We define a {\em path} for robot $i$ as a map 
$P_i: \mathbb {N} \to V$ where $\mathbb N$ is the set of non-negative integers. 
A feasible $P_i$ must be a sequence of vertices that connects $s_i$ and $g_i$: 
1) $P_i(0) = s_i$;
2) $\exists T_i \in \mathbb N$, s.t. $\forall t \geq T_i, P_i(t) = g_i$;
3) $\forall t > 0$, $P_i(t) = P_i(t - 1)$ or $(P_i(t), P_i(t - 1)) \in E$. 

With warehouse automation-like applications in mind, we work with $\mathcal G$ being $4$-connected grids, aiming to minimize the \emph{makespan}, i.e., $\max_i\{|{P}_i|\}$.
Unless stated otherwise, $\mathcal G$ is assumed to be an $m_1 \times m_2$ grid with $m_1 \ge m_2 $. Also, ``randomness'' in this paper always refers to uniform randomness. 
The version of \mpp we study is sometimes referred to as the \emph{one-shot} MAPF problem \cite{stern2019multi}. We mention that our results also translate to guarantees on the life-long setting \cite{stern2019multi}, which is briefly discussed in Sec.~\ref{sec:conclusion}. 

\subsection{The Rubik Table Problem (\rtp)}
The Rubik Table problem (\rtp)~\cite{szegedy2020rearrangement} formalizes the task of carrying out globally coordinated token swapping operations on lattices, with many interesting applications. The problem has many variations; in our study, we use the basic 2D form and the associated algorithms, which are summarized below.

\begin{problem}[{\normalfont \bf Rubik Table Problem (\rtp)} \cite{szegedy2020rearrangement}]
Let $M$ be an $m_1 (row) \times m_2 (column)$ table, $m_1 \ge m_2$, containing $m_1m_2$ items, one in each table cell. 
The $m_1m_2$ items are of $m_2$ colors with each color having a multiplicity of $m_1$.
In a \emph{shuffle} operation, the items in a single column or a single
row of $M$ may be permuted in an arbitrary manner. 
Given an arbitrary configuration $X_I$ of the items, find a sequence of shuffles that take $M$ from $X_I$ to the configuration where
row $i$, $1 \leq i \leq m_1$, contains only items of color $i$. The problem may also be \emph{labeled}, i.e., each item has a unique label in $1, \ldots, m_1m_2$.
\end{problem}

A key result from \cite{szegedy2020rearrangement}, which we denote as the Rubik Table Algorithm (\rta), establishes that a colored \rtp can be solved using $m_2$ column shuffles followed by $m_1$ row shuffles.
Additional $m_1$ row shuffles then solve the labeled \rtp.

\begin{theorem}[{\normalfont \bf Rubik Table Theorem \cite{szegedy2020rearrangement}}]\label{t:rta}
An arbitrary Rubik Table problem on an $m_1\times m_2$ table can be solved using $m_1 + m_2$ shuffles. The labeled Rubik Table problem can be solved using $2m_1 + m_2$ shuffles.
\end{theorem}

We briefly illustrate how \rta works on an $m_1 \times m_2$ table with $m_1 =4$ and $m_2 =3$ (in Fig.~\ref{fig:rubik}); we refer readers to \cite{szegedy2020rearrangement} for more details. \rta operates in two phases. In the first phase, a bipartite graph $B(T, R)$ is constructed based on the initial table configuration where the partite set $T$ are the colors/types of items, and the set $R$ are the rows of the table  (Fig.~\ref{fig:rubik}(b)). An edge is added to $B$ between $t \in T$ and $r \in R$ for every item of color $t$ in row $r$. 
From $B(T,R)$, a set of $m_2$ \emph{perfect matchings} can be computed, as guaranteed by \cite{hall2009representatives}. Each matching, containing $m_1$ edges, connects all of $T$ to all of $R$, and dictates how a column should look like after the first phase. For example, the first set of matching in solid lines in Fig.~\ref{fig:rubik}(b) says that the first column should be ordered as yellow, cyan, red, and green as shown in Fig.~\ref{fig:rubik}(c). 
After all matchings are processed, we get an intermediate table, Fig.~\ref{fig:rubik}(c). Notice that each row of Fig.~\ref{fig:rubik}(a) can be shuffled to yield the corresponding row of Fig.~\ref{fig:rubik}(c); this is the key novelty of the \rta.
After the first phase of $m_1$ row shuffles, the intermediate table (Fig.~\ref{fig:rubik}(c))  can then be rearranged with $m_2$ column shuffles to solve the colored \rtp (Fig.~\ref{fig:rubik}(d)). Another $m_1$ row shuffles can then solve the labeled \rtp (Fig.~\ref{fig:rubik}(e)). We note that it is also possible to perform labeled rearrangement using $m_2$ column shuffles followed by $m_1$ row shuffles and then followed by another $m_2$ column shuffles. 

\begin{figure}[h]
\vspace{-1mm}
        \centering
        \begin{overpic}[width=\linewidth]{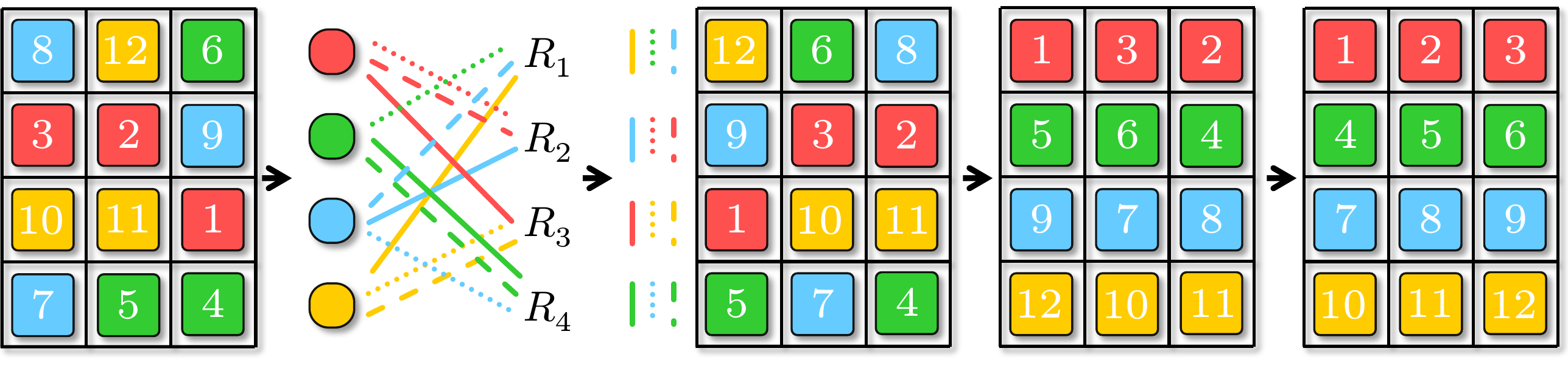}
             \footnotesize
             \put(6.5,  -3) {(a)}
             \put(26,  -3) {(b)}
             \put(50.5,  -3) {(c)}
             \put(70,  -3) {(d)}
             \put(89.5,  -3) {(e)}
        \end{overpic}
        
        \vspace{3.5mm}
        \caption{Illustration of applying the \emph{$11$ shuffles}. (a) The initial $4\times 3$ table with  a random arrangement of 12 items that are colored and labeled. The labels are consistent with the colors. (b) The constructed bipartite graph. It contains $3$ perfect matchings, determining the $3$ columns in (c); only color matters in this phase. (c) Applying $4$ row shuffles to (a), according to the matching results, leads to an intermediate table where each column has one color appearing exactly once. (d) Applying $3$ column shuffles to (c) solves a colored \rtp. (e) $4$ additional row shuffles fully sort the labeled items.} 
        \label{fig:rubik}
\vspace{-1mm}        
    \end{figure}
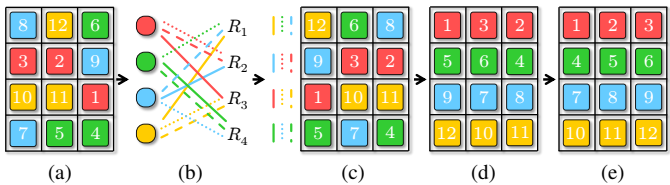

\rta runs in $O(m_1m_2\log m_1)$ (notice that this is nearly linear with respect to $n = m_1m_2$, the total number of items) expected time or $O(m_1^2m_2)$ deterministic time. 
If $m_1 = m_2 = m$, then the times become $O(m^2\log m)$ expected and $O(m^3)$ deterministic, respectively. 

\section{Solving \mpp up to Maximum Density w/ \rta}\label{sec:1:1}
The built-in global coordination capability of \rta naturally applies to solving makespan-optimal \mpp.
Since \rta only requires \emph{three rounds} of shuffles and each round involves either parallel row shuffles or parallel column shuffles, if each round of shuffles can be realized with makespan proportional to the size of the row/column, then a makespan upper bound of $O(m_1 + m_2)$ can be guaranteed. 
This is in fact achievable even when all of $\mathcal G$'s vertices are occupied by robots, by recursively applying a \emph{labeled line shuffle algorithm} \cite{yu2018constant}, which can arbitrarily rearrange a line of $m$ robots embedded in a grid using $O(m)$ makespan. 
\begin{lemma}[Basic Simulated Labeled Line Shuffle \cite{yu2018constant}]\label{l:line-shuffle} For $m$ labeled robots on a straight path of length $m$, embedded in a 2D grid, they may be arbitrarily ordered in $O(m)$ steps. Moreover, multiple such reconfigurations can be performed on parallel paths within the grid. 
\end{lemma}

The key operation is based on a localized, $3$-step pair swapping routine, shown in Fig.~\ref{fig:figure8}. For more details on the line shuffle routine, see \cite{yu2018constant}. 

\begin{figure}[h!]
        \centering
        \includegraphics[width=\linewidth]{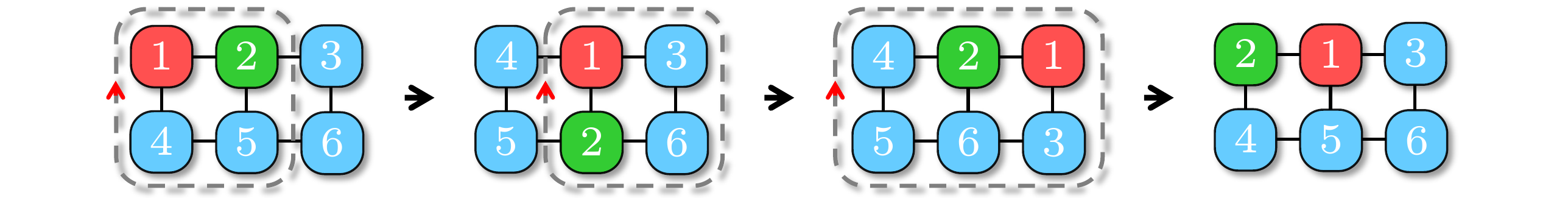}
\vspace{-3mm}
\caption{On a $2 \times 3$ grid, swapping two robots may be performed in three steps with three cyclic rotations.} 
        \label{fig:figure8}
\vspace{-4mm}
\end{figure}
    
The basic simulated labeled line-shuffle algorithm, however, has a large constant factor. Borrowing ideas from \emph{parallel odd-even sort} \cite{bitton1984taxonomy}, we can greatly reduce the constant factor in Lemma \ref{l:line-shuffle}. First, we need the following lemma. 

\begin{lemma}\label{l:reconfigure}
It takes at most seven steps and six steps to perform arbitrary combinations of pairwise horizontal swaps on $3\times 2$ grids and $4\times 2$ grids, respectively.
\end{lemma}
\ifarxiv
\begin{proof}[Proof of Lemma~\ref{l:reconfigure}]
Using integer programming \cite{yu2016optimal}, we exhaustively compute makespan-optimal solutions for arbitrary horizontal reconfiguration on $3\times 2$ ($8$ possible cases) and $4 \times 2$ grids ($16$ possible cases), which confirms the claim. 
\end{proof}
\fi

As an example, it takes seven steps to horizontally ``swap'' all three pairs of robots on a $3\times 2$ grid, as shown in Fig.~\ref{fig:six}. 

 \begin{figure}[h!]
 \vspace{-2mm}
        \centering
        \includegraphics[width=1\linewidth]{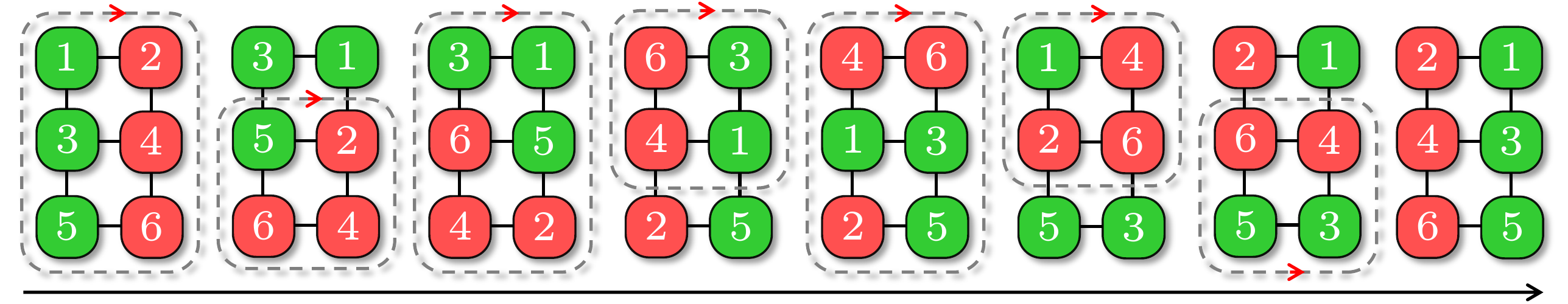}
\vspace{-3mm}
\caption{An example of a horizontal ``swap'' on a $3\times 2$ grid that takes seven steps, in which all three pairs are swapped. It takes at most seven steps to horizontally swap robots arbitrarily on a $3\times 2$ grid. } 
        \label{fig:six}
\vspace{-2mm}        
    \end{figure}

\def\rtmapf{\textsc{RTM}\xspace}
\begin{lemma}[Faster Line Shuffle]\label{l:fast-line-shuffle}
For $m$ robots on a straight path of length $m$, embedded in a 2D grid, they may be arbitrarily ordered in $7m$ steps. Moreover, multiple such reconfigurations can be performed simultaneously on parallel straight paths within the grid. 
\end{lemma}
\ifarxiv
\begin{proof}
``Sorting'' of $m$ robots on a straight path of length $m$ may be realized using  parallel odd-even sort \cite{bitton1984taxonomy} in $m-1$ rounds, which only requires the ability to simulate potential pairwise ``swaps'' interleaving odd phases (swapping robots located at positions $2k + 1$ and $2k + 2$ on the path for some $k$) and even phases (swapping robots located at positions $2k + 2$ and $2k + 3$ on the path for some $k$). Here, it does not matter whether $m$ is odd or even. 
To simulate these swaps, we can partition the grid embedding the path into $3 \times 2$ grids in two ways for the two phases, as illustrated in Fig.~\ref{fig:odd-even}. 
\begin{figure}[h!]
        \centering
        \includegraphics[width=1\linewidth]{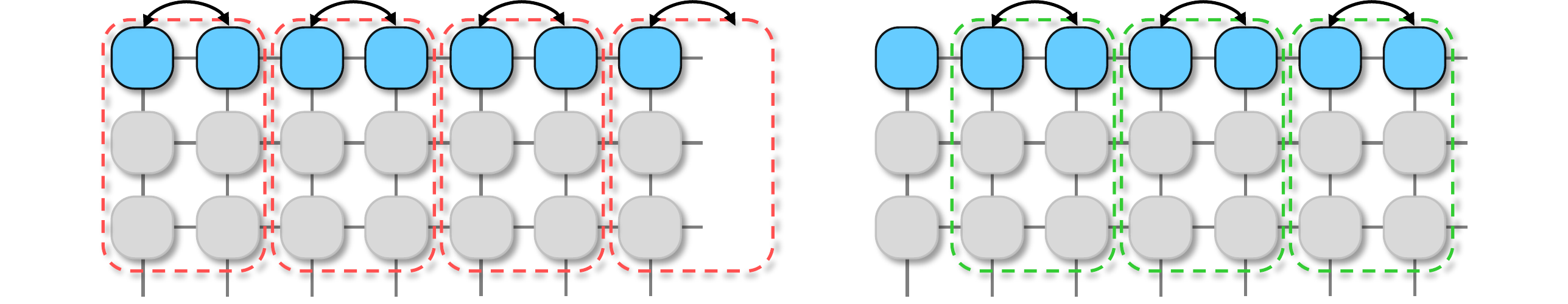}
\vspace{-3mm}
\caption{Partitioning a grid into disjoint $3 \times 2$ grids in two ways for simulating odd-even sort. The highlighted pairs of robots maybe independently ``swapped'' within each $3\times 2$ grid as needed.} 
        \label{fig:odd-even}
\vspace{-1.5mm}
    \end{figure}

A perfect partition requires that the second dimension of the grid, perpendicular to the straight path, be a multiple of $3$. If this is not the case, some partitions at the bottom can use $4 \times 2$ grids. By Lemma~\ref{l:reconfigure}, each odd-even sorting phase can be simulated using at most $7m$ steps. Clearly, shuffling on parallel paths is directly supported. 
\end{proof}
\fi

Combining \rta and fast line shuffle (Lemma~\ref{l:fast-line-shuffle}) yields a polynomial time \mpp algorithm for fully occupied grids with a makepsan of $7m_1 + 14m_2$. 

\begin{theorem}[\mpp on Grids under Maximum Robot Density, Upper Bound]\label{t:rtm-makespan}
\mpp on an $m_1\times m_2$ grid, $m_1 \ge m_2 \ge 3$, with each grid vertex occupied by a robot, can be solved in polynomial time in a makespan of $7m_1 + 14m_2$.
\end{theorem}

We note that the case of $m_2 = 2$ can also be solved similarly except when $m_1 = 2$, with a slightly altered procedure since we can only use partitions of $2 \times 3$ grids. We omit the details for this minor case which readers can readily fill in. 

The straightforward pseudo-code for RTM, the \rta based algorithm for \mpp on grids supporting the maximum possible robot density, is given in Alg.~\ref{alg:rubik}. The comments in the main \textsc{RTM} routine indicate the corresponding \rta phases. For \mpp on an $m_1 \times m_2$ grid with row-column coordinates $(x, y)$, we say robot $i$ belongs to color $1 \le j \le m_1$ if $g_i.y=j$.
Function $\texttt{Prepare()}$ in the first phase finds intermediate states $\{\tau_i\}$ for each robot through perfect matchings and routes them towards the intermediate states by (simulated) column shuffles. 
If the robot density is smaller than required, we may fill the table with ``virtual'' robots \cite{han2018sear,yu2018constant}.
For each robot $i$ we have $\tau_i.y=s_i.y$. 
Function $\texttt{ColumnFitting()}$ in the second phase routes the robots to their second intermediate states $\{\mu_i\}$ through row shuffles where $\mu_{i}.x=\tau_{i}.x$ and $\mu_i.y=g_i.y$. 
In the last phase, function $\texttt{RowFitting()}$ routes the robots to their final goal positions using additional column shuffles.

\begin{algorithm}
\begin{small}
\DontPrintSemicolon
\SetKwProg{Fn}{Function}{:}{}
\SetKwFunction{Fprepare}{Prepare}
\SetKwFunction{Fcolumnfitting}{ColumnFitting}
\SetKwFunction{Frowfitting}{RowFitting}
\SetKwFunction{FRtMapf}{RTM}

  \caption{Rubik Table Based \mpp Solver \label{alg:rubik}}
  \KwIn{Start and goal vertices $S=\{s_i\}$ and $G=\{g_i\}$}
  \Fn{\textsc{RTM}({$S,G$})}{
$\texttt{Prepare}(S,G)$ \quad\quad\quad\hspace{-0.5mm} \Comment{Computing Fig.~\ref{fig:rubik}(b)}\;
$\texttt{ColumnFitting}(S,G)$ \Comment{Fig.~\ref{fig:rubik}(a) $\to$ Fig.~\ref{fig:rubik}(c)}\;
$\texttt{RowFitting}(S,G)$ \quad\quad\hspace{-1.5mm} \Comment{Fig.~\ref{fig:rubik}(c) $\to$
Fig.~\ref{fig:rubik}(d)}\;
}
\vspace{1mm}
  \Fn{\Fprepare{$S,G$}}{
  
    $A\leftarrow [1,...,m_1m_2]$\;
  \For{$(t,r)\in [1,...,m_1]\times[1,...,m_1]$}{
  \If{$\exists i\in A$ where $s_i.x=r\wedge g_i.y=t$}
  {
 add edge $(t,r)$ to $B(T,R)$\;
 remove $i$ from $A$ \;
  }
    }
 compute matchings $\mathcal{M}_1,...,\mathcal{M}_{m_2}$ of $B(T, R)$\;
 $A\leftarrow [1,...,m_1m_2]$\;
 \ForEach{$\mathcal{M}_r$ and $(t,r)\in \mathcal{M}_r$}{
 \If{$\exists i\in A$ where $s_i.x=r\wedge g_i.y=t$}{
 $\tau_i\leftarrow (r, s_i.y)$ and remove $i$ from $A$\;
  mark robot $i$ to go to $\tau_i$\;
 }
 }
 perform simulated column shuffles in parallel 
 }
 \vspace{1mm}
  \Fn{\Fcolumnfitting{$S,G$}}{
        \ForEach{$i\in [1,...,m_1m_2]$}{
        $\mu_i\leftarrow (\tau_i.x,g_i.y)$ and mark robot $i$ to go to $\mu_i$\;
        }
  perform simulated row shuffles in parallel      
  }
\vspace{1mm}
  \Fn{\Frowfitting{$S,G$}}{
        \ForEach{$i\in [1,...,m_1m_2]$}{
        mark robot $i$ to go to $g_i$\;
        }
  perform simulated column shuffles in parallel 
  }
\end{small}
\end{algorithm}

We now establish the optimality guarantee of \rtmapf, assuming \mpp instances are randomly generated. For rearranging robots on an $m_1\times m_2$ grid, the expected makespan lower bound on random instances is $\Omega(m_1 + m_2)$ \cite{yu2018constant}. To obtain a finer optimality ratio, however, a finer lower bound is needed, which is established in the following. 

\begin{proposition}[Precise Makespan Lower Bound of \mpp on Grids]\label{p:makespan-lower}
The minimum makespan of random \mpp instances on an $m_1 \times m_2$ grid with $\Theta(m_1m_2)$ robots is $m_1 + m_2 - o(m_1)$ with arbitrarily high probability as $m_1\to \infty$.
\end{proposition}
\ifarxiv
\begin{proof}
Without loss of generality, let the constant in $\Theta(m_1m_2)$ be some $c > 0$, i.e., there are $cm_1m_2$ robots. 
We examine the top left and bottom right corners of the $m_1 \times m_2$ grid $\mathcal G$. In particular, let $\gtl$ (resp.,  $\gbr$) be the top left (resp., bottom right) $\alpha m_1\times \alpha m_2$ sub-grid of $\mathcal G$, for some  positive constant $\alpha \ll 1$.
For $u \in V(\gtl)$ and $v \in V(\gbr)$, assuming each grid edge has unit distance, then the Manhattan distance between $u$ and $v$ is at least $(1-2\alpha)(m_1 + m_2)$. 
Now, the probability that some $u  \in V(\gtl)$ and $v \in V(\gbr)$ are the start and goal, respectively, for a single robot, is $\alpha^4$. For $cm_1m_2$ robots, the probability that at least one robot's start and goal fall into $\gtl$ and $\gbr$, respectively, is $p =1 - (1 - \alpha^4)^{cm_1m_2}$. 

Because $(1 - x)^y < e^{-xy}$ for $0 < x < 1$ and $y > 0$ \footnote{This is because $\log(1-x) < -x$ for $0 < x < 1$; multiplying both sides by a positive $y$ and exponentiate with base $e$ then yield the inequality.}, $p > 1 - e^{-\alpha^4cm_1m_2}$. Therefore, for arbitrarily small $\alpha$, we may choose $m_1$ such that $p$ is arbitrarily close to $1$. 
For example, we may let $\alpha = m_1^{-\frac{1}{8}}$, which decays to zero as $m_1 \to \infty$, then it holds that the makespan is $(1 - 2\alpha)(m_1 + m_2) = m_1 + m_2 - 2m_1^{-\frac{1}{8}}(m_1 + m_2) = m_1 + m_2 - o(m_1)$ with probability $p > 1 - e^{-c\sqrt{m_1}m_2}$. 
\end{proof}
\fi

Comparing the upper bound established in Theorem~\ref{t:rtm-makespan} and the lower bound from Proposition~\ref{p:makespan-lower} immediately yields

\begin{theorem}[Optimality Guarantee of \rtmapf]\label{t:rtmapf}
For random \mpp instances on an $m_1 \times m_2$ grid with $\Omega(m_1m_2)$ robots, $m_1 \ge m_2 \ge 3$, as $m_1\to \infty$, \rtmapf computes in polynomial time solutions that are $7(1 + \frac{m_2}{m_1+m_2})$-makespan optimal, with high probability.
\end{theorem}

We emphasize that \rtmapf always runs in polynomial time and is not limited by any probabilistic guarantee; the high probability guarantee is only for solution optimality. The same is true for other algorithms' high probability guarantees proposed in this paper. We also note that high probability guarantees are stronger than and imply guarantees in expectation.

\section{Near-Optimally Solving \mpp with up to One Third Robot Density}\label{sec:1:2}
Though \rtmapf runs in polynomial time and provides constant factor makespan optimality in expectation, the constant factor is still relatively large due to the extreme density. In practice, a robot density of around $\frac{1}{3}$ (i.e., $n = \frac{m_1m_2}{3}$) is already very high. As it turns out,  with $n = cm_1m_2$ for some constant $c > 0$ and $n \le \frac{m_1m_2}{3}$, which is assumed throughout this section, the constant factor can be dropped significantly by employing a ``highway'' heuristic  to simulate the row/column shuffle operations.  

\subsection{Rubik Table for \mpp with ``Highway'' Shuffle Primitive}
\textbf{Random \mpp instances.}
For the highway heuristics, we first work with random \mpp instances.
Let us assume for the moment that $m_1$ and $m_2$ are multiples of three; we partition $\mathcal G$ into $3\times 3$ cells (see, e.g., Fig.~\ref{fig:jd_center}(b) and Fig.~\ref{fig:example}).
We use Fig.~\ref{fig:example}, where Fig.~\ref{fig:example}(a) is a random start configuration and Fig.~\ref{fig:example}(f) is a random goal configuration, as an example to illustrate \rth -- Rubik Table (for \mpp) with Highways, targeting robot density up to $\frac{1}{3}$. 
\rth involves two phases: \emph{anonymous reconfiguration} and \emph{\mpp resolution with Rubik Table and highway heuristics}. 

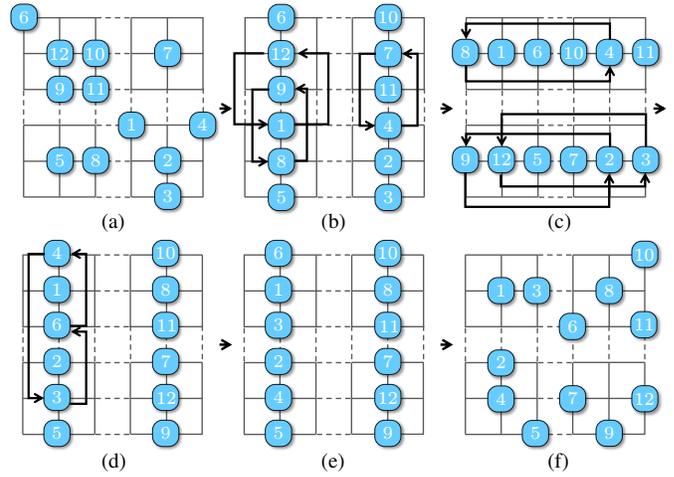
\begin{figure}[htbp]
\centering

        \begin{overpic}[width=\linewidth]{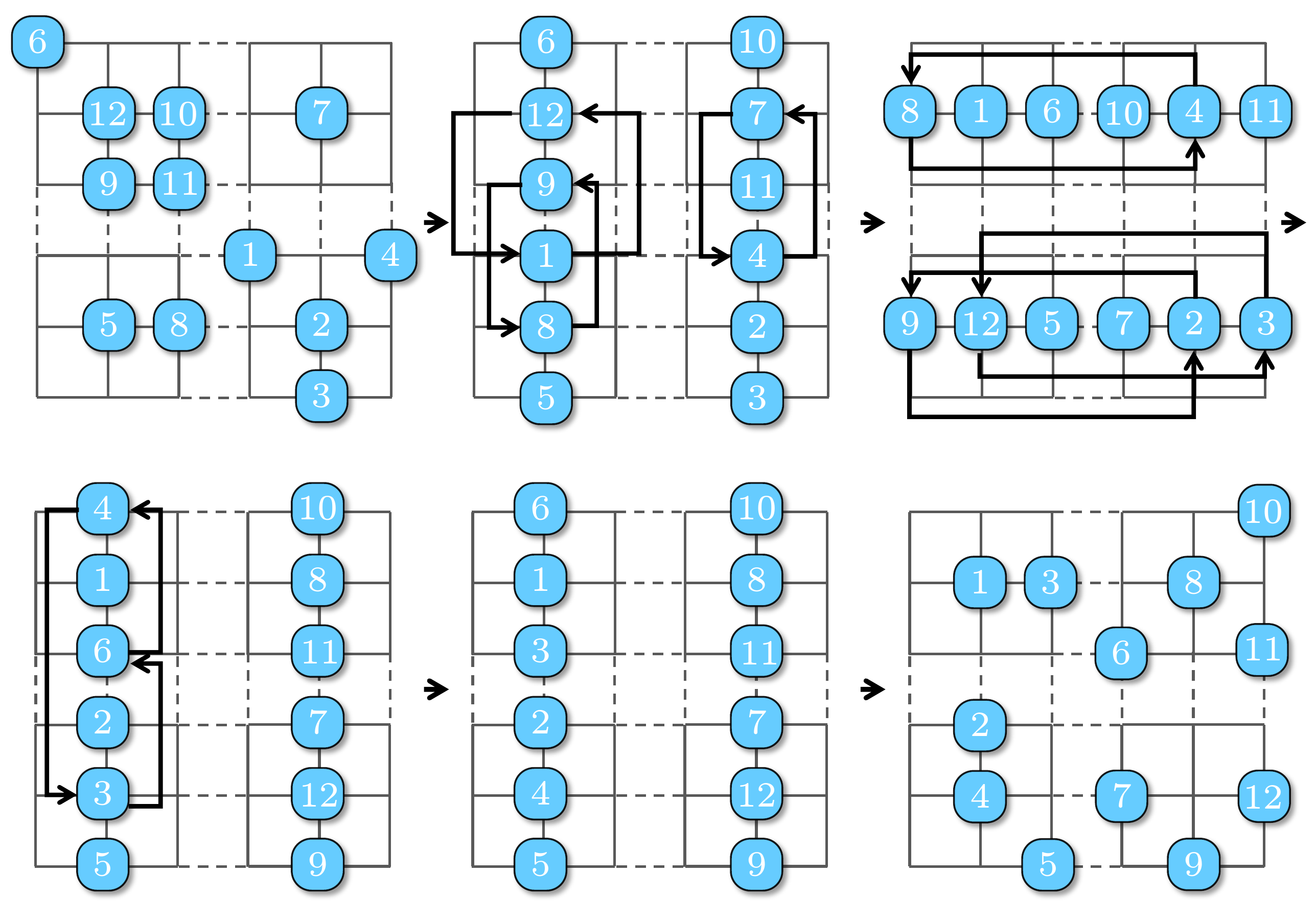}
             \footnotesize
             \put(14.5, 35) {(a)}
             \put(47.5, 35) {(b)}
             \put(81.5, 35) {(c)}
             \put(14.5, -1) {(d)}
             \put(47.5, -1) {(e)}
             \put(81.5, -1) {(f)}
        
        \end{overpic}
\vspace{-2mm}        
        \caption{An example of applying \rth to solve an \mpp instance. (a) The start configuration; (b) The start balanced configuration obtained from (a);  (c) The intermediate configuration obtained from the Rubik Table preparation phase; (d) The intermediate configuration obtained from the column fitting phase. Apply additional column shuffles for labeled items; (e) The goal balanced configuration obtained from the goal configuration; (f) The goal configuration.} 
        \label{fig:example}
\vspace{-2mm}        
    \end{figure}

In the anonymous reconfiguration phase, in which robots are treated as being indistinguishable or unlabeled, arbitrary start and goal configurations (under $\frac{1}{3}$ robot density) are converted to intermediate configurations where each $3 \times 3$ cell contains no more than $3$ robots. We call such configurations \emph{balanced configurations}. With high probability, random \mpp instances are not far from being balanced. To establish this result (Proposition~\ref{p:phase:1}), we need the following. 
\begin{theorem}[Minimax Grid Matching \cite{leighton1989tight}]\label{t:minimax}
Consider an $m \times m$ square containing $m^2$ points following the uniform distribution. Let $\ell$ be the minimum length such that there exists a perfect matching of the $m^2$ points to the grid points in the square for which the distance between every pair of matched points is at most $\ell$. Then $\ell = O(\log^{\frac{3}{4}}m)$ with high probability.
\end{theorem}

Theorem~\ref{t:minimax} applies to rectangles with the longer side being $m$ as well (Theorem 3 in \cite{leighton1989tight}). 

\begin{proposition}\label{p:phase:1}
On an $m_1\times m_2$ grid, with high probability, a random configuration of $n = \frac{m_1m_2}{3}$ robots is of distance $o(m_1)$ to a balanced configuration. 
\end{proposition}
\ifarxiv
\begin{proof}
We prove for the case of $m_1 = m_2 = 3m$ using the minimax grid matching theorem (Theorem~\ref{t:minimax}); generalization to $m_1 \ge m_2$ can be then seen to hold using the generalized version of Theorem~\ref{t:minimax} that applies to rectangles (Theorem 3 of \cite{leighton1989tight}, which in fact applies to arbitrarily simply connected region within a square region).

Now let $m_1 = m_2 = 3m$. We may view a random configuration of $m^2$ robots on a $3m\times 3m$ grid as randomly placing $m^2$ continuous points in an $m\times m$ square with scaling (by three in each dimension) and rounding. 
By Theorem~\ref{t:minimax}, a random configuration of $m^2$ continuous points in an $m \times m$ square can be moved to the $m^2$ grid points at the center of the $m^2$ disjoint unit squares within the $m \times m$ square, where each point is moved by a distance no more than $O(\log^{\frac{3}{4}}m)$, with high probability. 
Translating this back to a $3m \times 3m$ gird, we have that $m^2$ randomly distributed robots on the grid can be moved so that each $3\times 3$ cell contains exactly one robot and the maximum distance moved for any robot is no more than $O(\log^{\frac{3}{4}}m)$, with high probability. Applying this argument three times yields that a random configuration of $\frac{m_1^2}{3}$ robots on an $m_1\times m_1$ gird can be moved so that each $3\times 3$ cell contains exactly three robots and no robot needs to move more than a $O(\log^{\frac{3}{4}}{m_1})$ steps, with high probability. We note that, because the robots are indistinguishable, overlaying three sets of reconfiguration paths will not cause an increase in the distance traveled by any robot (and will reduce it). \end{proof}
\fi
In the example, anonymous reconfiguration corresponds to Fig.~\ref{fig:example}(a)$\to$Fig.~\ref{fig:example}(b) and Fig.~\ref{fig:example}(f)$\to$Fig.~\ref{fig:example}(e) (note that \mpp solutions are time-reversible). 
We simulated the process of anonymous reconfiguration for $m_1 = m_2 = 300$, i.e., on a $300 \times 300$ grids. For $\frac{1}{3}$ robot density, the actual number of steps, averaged over $100$ random instances, is less than $5$.
We call configurations like  Fig.~\ref{fig:example}(b)-(e), which have all robots concentrated vertically or horizontally in the middle of the $3\times 3$ cells, \emph{centered balanced configurations} or simply \emph{centered configurations}. 
Completing the first phase requires solving two unlabeled \mpp problems \cite{YuLav12CDC,Ma2016OptimalTA}, easily doable in polynomial time. 

In the second phase, \rta is applied with a highway heuristic to get us from Fig.~\ref{fig:example}(b) to Fig.~\ref{fig:example}(e), transforming between vertical centered configurations and horizontal centered configurations. 
To do so, \rta is applied (e.g., to Fig.~\ref{fig:example}(b) and (e)) to obtain two intermediate configurations (e.g., Fig.~\ref{fig:example}(c) and (d)).
To go between these configurations, e.g., Fig.~\ref{fig:example}(b)$\to$Fig.~\ref{fig:example}(c), we apply a heuristic by moving robots that need to be moved out of a $3\times 3$ cell to the two sides of the middle columns of Fig.~\ref{fig:example}(b), depending on their target direction. If we do this consistently, after moving robots out of the middle columns, we can move all robots to their desired goal $3\times 3$ cell without stopping nor collision. 
Once all robots are in the correct $3\times 3$ cells, we can convert the balanced configuration to a centered configuration in at most $3$ steps, which is necessary for carrying out the next simulated row/column shuffle. 
Adding things up, we can simulate a shuffle operation using no more than $m + 5$ steps where $m = m_1$ or $m_2$. 
The efficient simulated shuffle leads to low makespan \mpp routing algorithms. It is clear that all operations take polynomial time; a precise running time is given at the end of this subsection.

\begin{theorem}[Makespan Upper Bound for Random \mpp, $\le \frac{1}{3}$ Density]\label{t:rtm-ramdom}
For random \mpp instances on an $m_1 \times m_2$ grid, where $m_1 \ge m_2$ are multiples of three, for $n \le \frac{m_1m_2}{3}$ robots, an $m_1 + 2m_2 + o(m_1)$ makespan solution can be computed in polynomial time, with high probability. 
\end{theorem}
\ifarxiv
\begin{proof}
By Proposition~\ref{p:phase:1}, anonymous reconfiguration requires distance $o(m_1)$ with high probability. By Theorem 1 from \cite{yu2018constant}, this implies that a plan can be obtained for anonymous reconfiguration that requires $o(m_1)$ makespan.
For the second phase of \mpp resolution with Rubik Table and highway heuristics, by Theorem~\ref{t:rta}, we need to perform $m_1$ parallel row shuffles with row width of $m_2$, followed by $m_2$ parallel column shuffles with column width of $m_1$, followed by another $m_1$ parallel row shuffles with row width of $m_2$. Simulating these shuffles require $m_1 + 2m_2 + O(1)$ steps. All together, a makespan of $m_1 + 2m_2 + o(m_1)$ is required, with very high probability. 
\end{proof}
\fi

Contrasting Theorem~\ref{t:rtm-ramdom} and Proposition~\ref{p:makespan-lower} yields

\begin{theorem}[Makespan Optimality for Random \mpp, $\le \frac{1}{3}$ Density]\label{t:rth-ratio}
For random \mpp instances on an $m_1 \times m_2$ grid, where $m_1 \ge m_2$ are multiples of three, for $n = cm_1m_2$ robots with $c \le \frac{1}{3}$, as $m_1 \to \infty$, a $(1 + \frac{m_2}{m_1 + m_2})$ makespan optimal solution can be computed in polynomial time, with high probability. 
\end{theorem}

Since $m_1\ge m_2$, $1 + \frac{m_2}{m_1 + m_2} \in (1, 1.5]$. In other words, in polynomial running time, \rth achieves $(1 + \delta)$ asymptotic makespan optimality for $\delta \in (0, 0.5]$, with high probability. 

From the analysis so far, if $m_1$ and/or $m_2$ are not multiples of $3$, it is clear that all results in this subsection continue to hold for robot density $\frac{1}{3} - \frac{(m_1\mod 3)(m_2\mod 3)}{m_1m_2}$, which is arbitrarily close to $\frac{1}{3}$ for large $m_1$. It is also clear that the same can be said for grids with certain patterns of regularly distributed obstacles (Fig.~\ref{fig:jd_center}(b)), i.e., 

\begin{corollary}[Random \mpp, $\frac{1}{9}$ Obstacle and $\frac{2}{9}$ Robot Density]
For random \mpp instances on an $m_1 \times m_2$ grid, where $m_1 \ge m_2$ are multiples of three and there is an obstacle at coordinates $(3k_1 + 2, 3k_2 + 2)$ for all applicable $k_1$ and $k_2$, for $n = cm_1m_2$ robots with $c \le \frac{2}{9}$, a solution can be computed in polynomial time that has makespan $m_1 + 2m_2 + o(m_1)$ with high probability.
As $m_1 \to \infty$, the solution approaches $1 + \frac{m_2}{m_1 + m_2}$ optimal, with high probability. 
\end{corollary}

\textbf{Arbitrary \mpp instances.}
We now examine applying \rth to arbitrary \mpp instances under $\frac{1}{3}$ robot density.
If an \mpp instance is arbitrary, all that changes to \rth is the makespan it takes to complete the anonymous reconfiguration phase. On an $m_1\times m_2$ grid, by computing a matching, it is straightforward to show that it takes no more than $m_1 + m_2$ steps to complete the anonymous reconfiguration phase, starting from an arbitrary start configuration.  Since two executions of anonymous reconfiguration are needed, this adds $2(m_1 + m_2)$ additional makespan. Therefore, we have

\begin{theorem}[Arbitrary \mpp, $\le \frac{1}{3}$ Density]
For arbitrary \mpp instances on an $m_1 \times m_2$ grid, $m_1 \ge m_2$, for $n \le \frac{m_1m_2}{3}$ robots, a $3m_1 + 4m_2 + o(m_1)$ makespan solution can be computed in polynomial time. This implies that, for $n = cm_1m_2 \le \frac{m_1m_2}{3}$ robots, a polynomial time can compute an asymptotic $3 + \frac{m_2}{m_1+m_2}$ makespan optimal solution, with high probability. 
\end{theorem}

We now give the running time of \rtmapf and \rth. 

\begin{proposition}[Running Time, \rth]\label{p:time}
For $n \le \frac{m_1m_2}{3}$ robots on an $m_1 \times m_2$ grid, \rth runs in $O(nm_1^2m_2)$ time.
\end{proposition}
\ifarxiv
\begin{proof}
The running time of \rtmapf and \rth are dominated by the matching computation and solving anonymous \mpp. 
The matching part takes $O(m_1^2m_2)$ in deterministic time or $O(m_1m_2\log m_1)$ in expected time \cite{goel2013perfect}. 
Anonymous \mpp may be tackled using the max-flow algorithm \cite{ford1956maximal} in $O(nm_1m_2T)=O(nm_1^2m_2)$ time, where $T=O(m_1+m_2)$ is the expansion  time  horizon of a time-expanded graph that allows a routing plan to complete. 
\end{proof}
\fi

\subsection{Reducing Makespan via Optimizing Matching}\label{subsec:heuristics}
Based on \rta, \rth has three simulated shuffle phases. Eventually, the makespan is dominated by the robot that needs the longest time to move, as a sum of moves for the robot in all three phases. As a result, the optimality of Rubik Table  methods is determined by the first preparation phase. The matchings determine the intermediate states in all three phases. 
Finding arbitrary perfect matchings is fast but the process can be improved to reduce the overall makespan. 

For improving matching, we propose two heuristics; the first is based on \emph{integer programming} (IP). 
We create binary variables $\{x_{ri}\}$ where $r$ represents the row number and $i$ the robot. 
robot $i$ is assigned to row $r$ if $x_{ri}=1$. 
Define single robot cost as $C_{ri}(\lambda)=\lambda |r-s_i.x|+(1-\lambda)|r-g_i.x|$. 
We optimize the makespan lower bound of the first phase by letting $\lambda=0$ or the third phase by letting $\lambda=1$. 
The objective function and constraints are given by
\begin{equation}
\label{eq:objective}
    \max_{r,i} \{{C_{ri}(\lambda=0)x_{ri}}\}+\max_{r,i}\{C_{ri}(\lambda=1)x_{ri}\}
\end{equation}
\vspace{-2mm}
\begin{equation}
\label{eq:constraint1}
    \sum_{r}x_{ri}=1, \text{for each robot $i$}
\end{equation}
\vspace{-2mm}
\begin{equation}
\label{eq:constraint2}
    \sum_{g_i.y=t}x_{ri}\leq 1, {\small\text{for each row $r$ and each color $t$}}  
\end{equation}
\vspace{-1mm}
\begin{equation}
\label{eq:constraint3}
\sum_{s_i.y=c}x_{ri}=1, \text{for each column $c$ and each row $r$}
\end{equation}

Eq. \eqref{eq:objective} is the summation of makespan lower bound of the first phase and the third phase. Note that the second phase cannot be improved through optimizing the matching.
Eq. \eqref{eq:constraint1} requires that robot $i$ be only present in one row.
Eq. \eqref{eq:constraint2} specifies that each row should contain robots that have different goal columns.
Eq. \eqref{eq:constraint3} specifies that each vertex $(r,c)$ can only be assigned to one robot. 
The IP model represents a general assignment problem which is NP-hard in general.
It has limited scalability but provides a way to evaluate how optimal the matching could be in the limit.

A second matching heuristic we developed is based on \emph{linear bottleneck assignment (LBA)} \cite{burkard2012assignment}, which takes polynomial time.
LBA differs from the IP heuristic in that the bipartite graph is weighted.
For the matching assigned to row $r$, the edge weight of the bipartite graph is computed greedily.
If column $c$ contains robots of color $t$, we add an edge $(c,t)$ and its edge cost is
\vspace{-1.5mm}
\begin{equation}
\vspace{-1.5mm}
    C_{ct}=\min_{g_i.y=t}C_{ri}(\lambda=0)
\end{equation}
We choose $\lambda=0$ to optimize the first phase. Optimizing the third phase ($\lambda=1$) would give similar results.
After constructing the weighted bipartite graph, an $O(\frac{m_1^{2.5}}{\log m_1})$ LBA algorithm \cite{burkard2012assignment} is applied to get a minimum bottleneck cost matching for row $r$. Then we remove the assigned robots and compute the next minimum bottleneck cost matching for next row. 
After getting all the matchings $\mathcal{M}_r$, we can further use LBA to assign $\mathcal{M}_r$ to a different row $r'$ to get a smaller makespan lower bound. The cost for assigning matching $\mathcal{M}_r$ to row $r'$ is defined as 
\vspace{-1.5mm}
\begin{equation}
\vspace{-1.5mm}
    C_{\mathcal{M}_rr'}=\max_{i\in\mathcal{M}_r}C_{r'i}(\lambda=0)
\end{equation}
The total time-complexity of using LBA heuristic for matching is $O(\frac{m_1^{3.5}}{\log m_1})$.

We denote \rth with IP and LBA heuristics as \rthip and \rthlba, respectively.
We mention that \rtmapf, which uses the line swap motion  primitive, can also benefit from these heuristics to re-assign the goals within each group. This can lower the bottleneck path length and improve the optimality.

\section{Near-Optimally Solving \mpp at Half Robot Density}\label{sec:1:3}
The key design philosophy behind \rtmapf and \rth is to effectively simulate row/column shuffles. With this in mind, we further explored the case of $\frac{1}{2}$ robot density. 
Using a more sophisticated shuffle routine, $\frac{1}{2}$ robot density can be supported while retaining most of the guarantees for the $\frac{1}{3}$ density setting; obstacles are no longer supported. 

To best handle $\frac{1}{2}$ robot density, we employ a new shuffle routine called \emph{linear merge}, based on merge sort, and denote the resulting algorithm as Rubik Table with Linear Merge heuristics or \rtm. 
The basic idea behind linear merge (shuffle) is straightforward: for $m$ robots on a $2 \times m$ grid,  we iteratively sort the robots first on $2\times 2$ grids, then $2\times 4$ grids, and so on, much like how merge sort works. An illustration of the process on a $2 \times 8$ grid is shown in Fig.~\ref{fig:merge}.

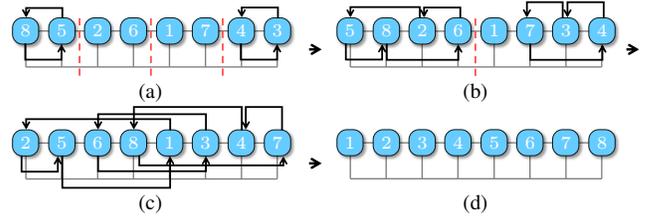
\begin{figure}[h!]
\vspace{-2mm}
\centering
    \begin{overpic}[width=\linewidth]{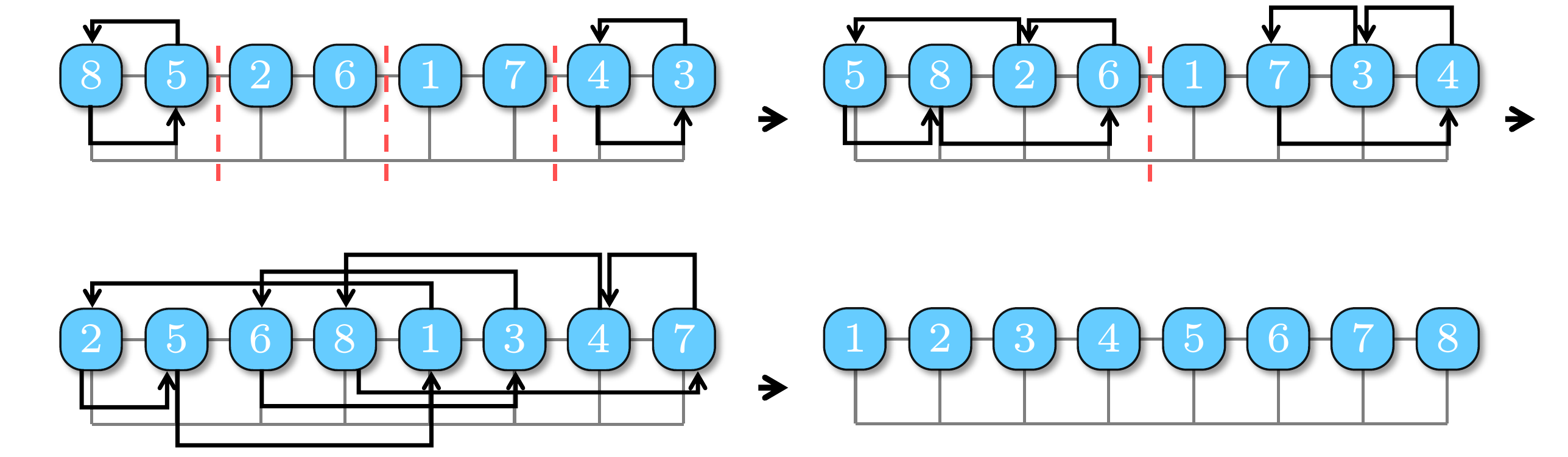}
             \footnotesize
             \put(22.8, 14.2) {(a)}
             \put(71.5, 14.2) {(b)}
             \put(22.8, -2.5) {(c)}
             \put(71.5, -2.5) {(d)}
    \end{overpic}
    \vspace{-2mm}
    \caption{A demonstration of the linear merge shuffle primitive on a $2\times 8$ grid. Robots going to the left always use the upper channel while robots going to the right always use the lower channel.} 
    \label{fig:merge}
    \vspace{-4mm}
\end{figure}

We now show that linear merge is always feasible and has the desired properties. 

\begin{lemma}[Properties of Linear Merge]\label{l:lm}
On a $2\times m$ grid, $m$ robots, starting on the first row, can be arbitrarily ordered using $m + o(m)$ steps. The motion plan can be computed in polynomial time. 
\end{lemma}
\ifarxiv
\begin{proof}
We first show \emph{feasibility}. The procedure takes $\lceil \log m \rceil$ phases; in a phase, let us denote a section of the $2\times m$ grid where robots are treated together as a \emph{block}. For example, the left $2 \times 4$ grid in Fig.~\ref{fig:merge}(b) is a block. It is clear that the first phase, involving up to two robots per block, is feasible (i.e., no collision). Assuming that phase $k$ is feasible, we look at phase $k + 1$. We only need to show that the procedure is feasible on one block of length up to $2^{k+1}$. For such a block, the left half block of length up to $2^k$ is already fully sorted as desired, e.g., in increasing order from left to right. For the $k+1$ phase, all robots in the left half block may only stay in place or move to the right. For these robots that stay, they must be all at the leftmost positions of the half block and will not block motions of any other robot. For the robots that do need to move to the right, their relative orders do not need to change, and therefore will not cause collisions among themselves. Because these robots that move in the left half block will move down on the grid by one edge, they will not interfere with any robot from left from the right half block. 
Because the same arguments hold for the right half block (except the direction change), the overall process of merging a block occurs without collision. 

Next, we examine the \emph{makespan}. For any single robot $r$, at phase $k$, suppose it belongs to block $b$ and block $b$ is to be merged with block $b'$. It is clear that the robot cannot move more than $len(b') + 2$ steps, where $len(b')$ is the number of columns of $b'$ and the $2$ extra steps may be incurred because the robot needs to move down and then up the grid by one edge. This is because any move that $r$ needs to do is to allow robots from $b'$ to move toward $b$.
Because there are no collisions in any phase, adding up all the phases, no robot moves more than $m + 2(\log m + 1) = m + o(m)$ steps. 

Finally, it is clear that the merge sort-like linear merge shuffle primitive runs in $O(m\log m)$ time since it is a standard divide-and-conquer routine with $\log m$ phases. 
\end{proof}
\fi

With linear merge, the asymptotic properties of \rth for $\frac{1}{3}$ robot density mostly carries over to \rtm. 

\begin{theorem}[Random \mpp, $\frac{1}{2}$ Robot Density]
For random \mpp instances on an $m_1 \times m_2$ grid, where $m_1 \ge m_2$ are multiples of two, for $\frac{m_1m_2}{3} \le n  \le \frac{m_1m_2}{2}$ robots, a solution can be computed in polynomial time that has makespan $m_1 + 2m_2 + o(m_1)$ with high probability.
As $m_1 \to \infty$, the solution approaches an optimality of $1 + \frac{m_2}{m_1 + m_2} \in (1, 1.5]$, with high probability. 
\end{theorem}

\section{Simulation Experiments}\label{sec:eval}
In  this  section,  we  evaluate Rubik Table based algorithms and compare them with fast and near-optimal solvers, \ecbs($w$=1.5) \cite{barer2014suboptimal} and \ddm \cite{han2020ddm}. 
These two methods are, to our knowledge, two of the fastest near-optimal solvers for \mpp.
We considered a state-of-the-art polynomial algorithm, push- and-swap \cite{luna2011push}, which gave fairly suboptimal results; the makespan optimality ratio is often above 100 for densities we examine. 
We also tested prioritized methods, e.g., \cite{ma2019searching,Okumura2019PriorityIW,silver2005cooperative}, which faced significant difficulties in resolving deadlocks. Given the limited relevance and considering the amount of results we are presenting, we omit these methods in our comparison.

All experiments are performed on an Intel\textsuperscript{\textregistered} Core\textsuperscript{TM} i7-9700 CPU at 3.0GHz. Each data point is an average over 20 runs on randomly generated instances, unless otherwise stated.
A running time limit of $300$ seconds is imposed over all instances. 
The optimality ratio is estimated as compared to conservatively estimated makespan lower bounds.
%
The Rubik Table based algorithms are implemented in Python. The compared solvers are C++ based.
As such, one can expect additional significant running time reductions from our algorithms implemented in C++. 
We choose Gurobi \cite{gurobi} as the mixed integer programming solver and  ORtools \cite{ortools} as the max-flow solver. 
The video of the simulations  can   be   found  at \href{https://youtu.be/aphCjWFwfss}{https://youtu.be/aphCjWFwfss}.

\subsection{Optimality of \rtmapf, \rtm, and \rth}
We first evaluate the optimality achieved by \rtmapf, \rtm, and \rth over randomly generated instances on their maximum designed robot density. That is, for \rtmapf, the grid is fully occupied; for \rtm and \rth, the robot density is $\frac{1}{2}$ and $\frac{1}{3}$, respectively. We test over three $m_1:m_2$ ratios: $1:1$, $3:2$, and $5:1$. The result is plotted in Fig.~\ref{fig:RTM-RTLM-RTH}. Computation time is not listed (because we list the computation time later for \rth; the running times of \rtmapf, \rtm, and \rth are similar). The optimality ratio is computed as the ratio between the solution makespan and the longest Manhattan distance between any pair of start and goal. Therefore, the estimate is an overestimate (i.e., the actual ratio may be lower/better). 
\begin{figure}[htbp]
\vspace{-1.5mm}
        \centering
         \includegraphics[width=1\linewidth]{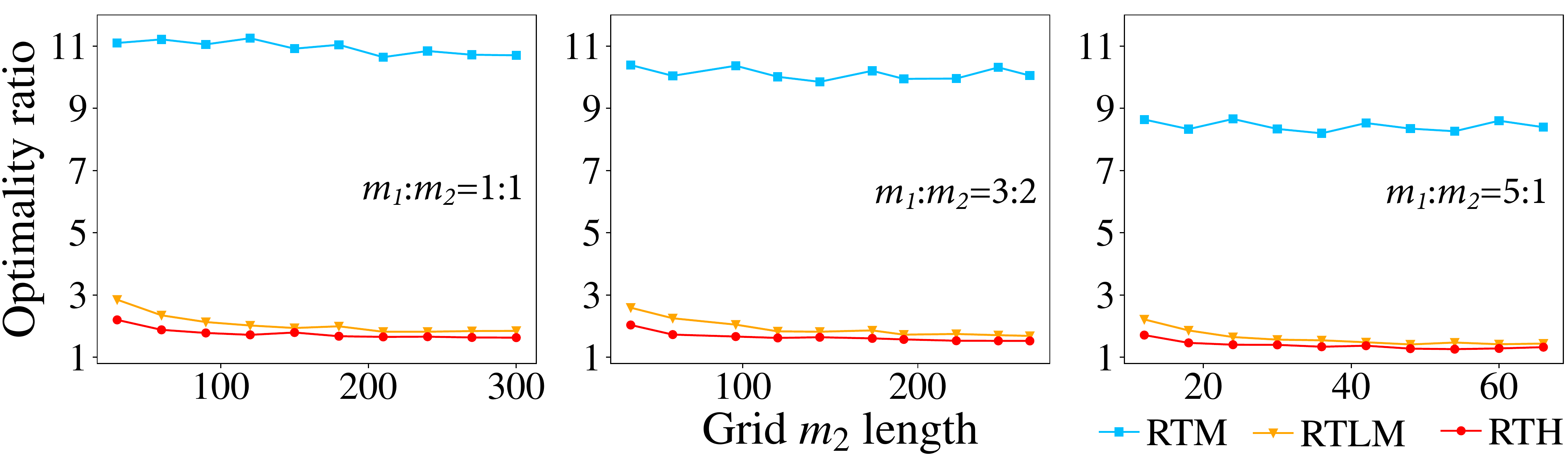}
\caption{Makespan optimality ratio for \rtmapf, \rtm, and \rth for their maximum designed robot density, for different grid sizes and $m_1:m_2$ ratios. We note that the largest problem has $90,000$ robots on a $300 \times 300$ grid.} 
        \label{fig:RTM-RTLM-RTH}
\vspace{-1.5mm}
    \end{figure}

We observe that \rtmapf achieves $7$--$10.5+$ makespan optimality ratio, which justifies the correctness of Theorem~\ref{t:rtmapf}.   
Both \rtm and \rth achieve sub-2 optimality guarantee for most of the test cases, with result for \rth dropping below $1.5$ on large grids. For all settings, as the grid size increases, there is a general trend of improvement of optimality across all methods/grid aspect ratios. This is due to two reasons: (1) the overhead in the shuffle operations becomes relatively smaller as grid size increases, and (2) with more robots, the makespan lower bound becomes closer to $m_1 + m_2$. Lastly, as $m_1:m_2$ ratio increases, the optimality ratio improves as predicted. For many test cases, the optimality ratio for the \rth for $m_1:m_2=5$ setting is around $1.3$.

We note that the performance of \rtm and \rth on optimality can be further improved using the heuristics described in Sec.~\ref{subsec:heuristics}. For the rest of the evaluations, we focus on \rth and its variants with additional heuristics. 

\subsection{Evaluation and Comparative Study of \rth}
\subsubsection{Impact of grid size}
For our first detailed comparative study of the performance of \rth, we set $m_1:m_2 = 3:2$ and fix robot density at $\frac{1}{3}$. In Fig.~\ref{fig:grid-size}, we compare the performance of ECBS\cite{barer2014suboptimal}, DDM\cite{han2020ddm}, \rth, \rthip, and \rthlba, in terms of computation time and optimality ratio. 

    \begin{figure}[htbp]
\vspace{-2mm}    
        \centering
        \includegraphics[width=1\linewidth]{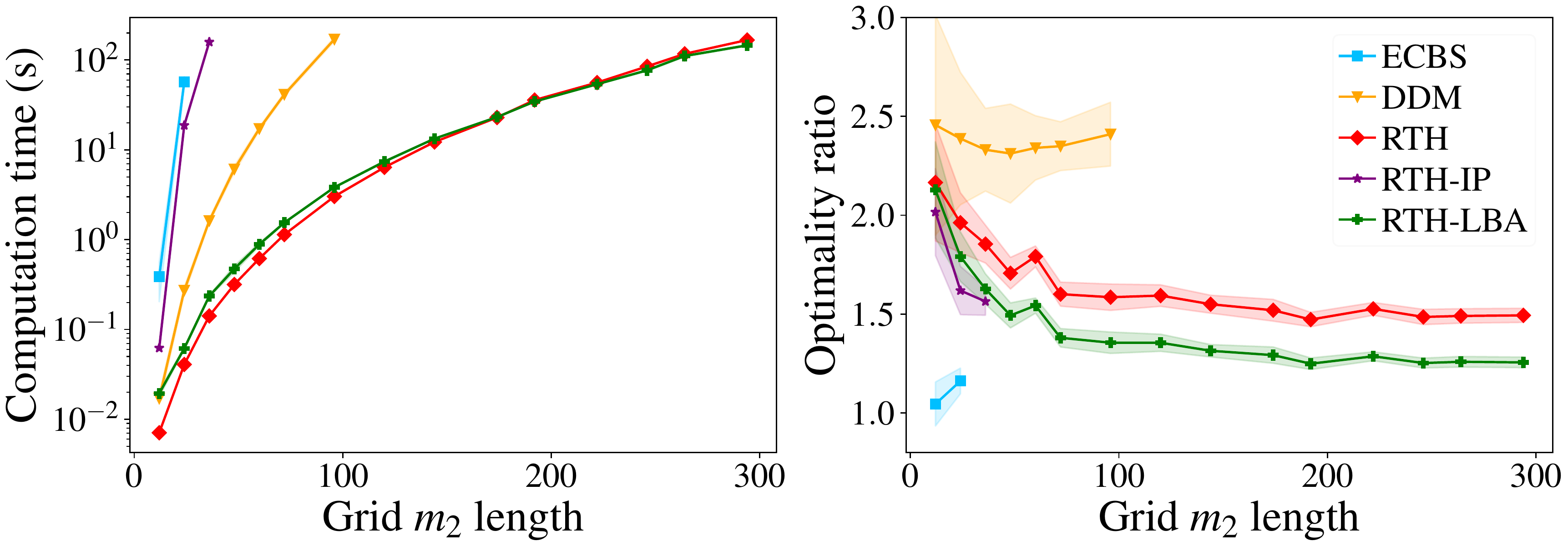}
        \vspace{-6mm}
        \caption{Computation time and optimality ratios on $m_1 \times m_2$ grids of varying sizes with $m_1:m_2 = 3:2$ and robot density at $\frac{1}{3}$.  One standard deviation is shown as shaded regions.} 
        \label{fig:grid-size}
        \vspace{-3mm}
\vspace{-1.5mm}        
    \end{figure}
    
Despite the less efficient Python-based implementation, \rtmapf, \rth, and \rthlba are faster than \ecbs and \ddm, due to their low polynomial running time guarantees. 
\rth and \rthlba can solve very large instances, e.g., on $450\times 300$ grids with $45,000$ robots in about $100$ seconds while neither \ecbs nor \ddm can.
\ecbs stopped working after $m_2 = 30$, though it shows better optimality on problems it can solve.
\ddm could handle up to $m_2 = 90$ but demonstrated poor optimality under the $\frac{1}{3}$ density setting.
%
The optimality ratio of \rth and \rthlba improves as the graph size increases, as predicted by our theoretical analysis.
The optimality ratios of \rth and \rthlba reach as low as $1.49$ and $1.26$, respectively, agreeing with the ratio predicted by  Theorem~\ref{t:rth-ratio}; here, the high probability asymptotic ratio is $1 + \frac{m_2}{m_1+m_2} = 1.4$. \rthlba is able to do better than $1.4$ because of the LBA heuristic. 
\rthip does slightly better on optimality in comparison to \rth and \rthlba, but its scalability is limited. 
On the other hand, \rthlba does nearly as well on optimality and remains competitive in terms of running time in comparison to \rth. 
As a consequence, we do not include further evaluation of \rthip.

For each method, the one standard deviation range is also shown in the figure. Because \rth and \rthlba are mostly deterministic and there are many robots, the change in optimality across different instances is small. We omit the inclusion of standard deviations from other plots as they are mostly similar in other tested settings. 

We mention that we also evaluated \ecbs with temporal splitting heuristics \cite{guo2021spatial}, which did not show significant difference in comparison to \ecbs at the density we tried; so we did not include it here. We also evaluated push-and-swap \cite{luna2011push}, which runs fast but yields very poor optimality ratios ($>100$ for many instances). We further evaluated prioritized methods, e.g., \cite{Okumura2019PriorityIW}, which faced significant difficulties in resolving deadlocks. Given these, we did not include results from these methods in our comparatively study.

\subsubsection{Impact of robot density}
Next, we experiment the impact of different robot density on a $180\times 120$ grid (Fig. \ref{fig: density}). At density $\frac{1}{3}$, there are up to $7,200$ robots. 
For densities below $\frac{1}{3}$, we add ``virtual robots'' when the matching step is performed. 
\ecbs does not appear because it cannot solve problems at this scale within $300$ seconds. 
    \begin{figure}[htbp]
        \centering
        \vspace{-2mm}
        \includegraphics[width=1\linewidth]{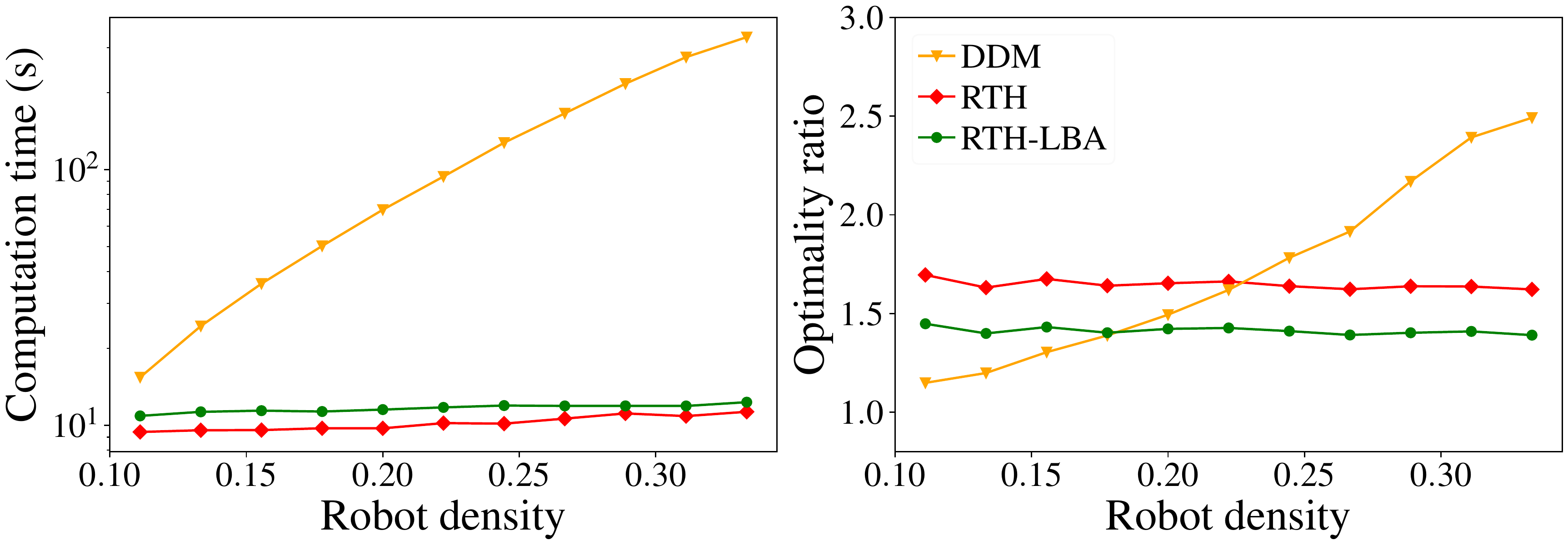}
        \vspace{-5mm}
        \caption{Computation time and optimality ratios on $180\times 120$ grids with varying robot density.} 
        \label{fig: density}
        \vspace{-3mm}
    \end{figure}

\rth and \rthlba both return solutions around $10$s on this graph for all instances. 
Both computation time and optimality ratio of \ddm grow as the robot density increases, while the robot density has little impact on \rth and \rthlba. 
At lower density, \ddm demonstrates better optimality. 
In environments with high densities, however, robots are highly coupled which causes more conflicts for structurally-agnostic approaches like \ddm (and \ecbs).
\rth and \rthlba show improved optimality as the density increases, reaching $1.4$, mostly due to the makespan of the instances getting larger.

\subsubsection{Handling obstacles}
\rth can also handle scattered obstacles and are especially suitable for cases where obstacles are regularly distributed. 
For instance, problems with underlying graphs like that in Fig. \ref{fig:jd_center}(b), where each $3 \times 3$ cell has a hole in the middle, can be natively solved without performance degradation.
Such settings find real-world applications in parcel sorting facilities in large warehouses \cite{wan2018lifelong,li2020lifelong}.
For this parcel sorting setup, we fix the robot density at $\frac{2}{9}$ and test \ecbs, \ddm, \rth and \rthlba on graphs with varying sizes. 
The results are shown in Fig \ref{fig: sorting_random}.
Note that \ddm can only apply when there is no narrow passage. So we added additional ``borders'' to the map to make it solvable for \ddm.
The results are similar as earlier ones; \rth and \rthlba run very fast and produce high-quality solutions, with conservatively estimated optimality ratio approaching 1.27.
   \begin{figure}[htbp]
        \centering
        \includegraphics[width=1\linewidth]{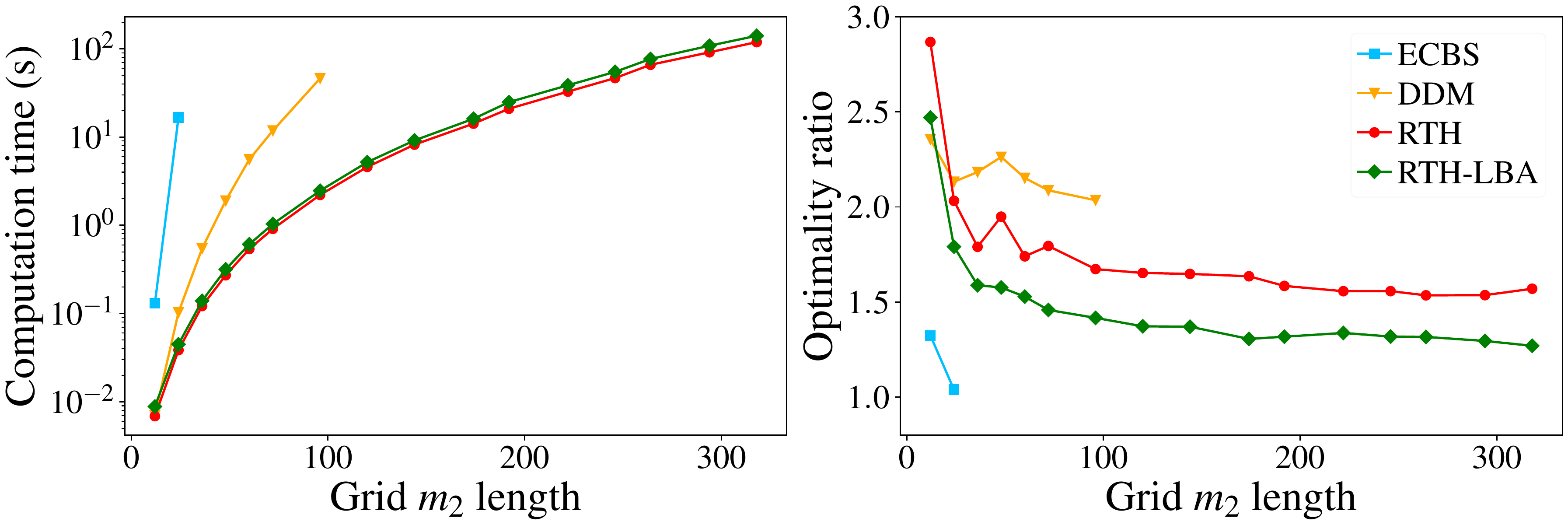}
        \vspace{-4mm}
        \caption{Computation time and optimality ratios on environments  of varying sizes with regularly distributed obstacles at $\frac{1}{9}$ density and robots at $\frac{2}{9}$ density. $m_1:m_2 = 3:2$.
        } 
        \label{fig: sorting_random}
        \vspace{-2mm}
    \end{figure}

\subsubsection{Impact of grid aspect ratios}
In this section, we fix $m_1m_2=90000$ and vary the $m_2:m_1$ ratio between $0$ (nearly one dimensional) and $1$ (square grids). We evaluated four algorithms, two of which are \rth and \rthlba. Now recall that \rtp on an $m_1 \times m_2$ table can also be solved using $2m_2$ column shuffles and $m_1$ row shuffles. Adapting \rth and \rthlba with $m_1 + 2m_2$ shuffles gives the other two variants which we denote as \rth-LL and \rthlba-LL respectively, with ``LL'' suggesting two sets of longer shuffles are performed (each set of column shuffle work with columns of length $m_1$). The result is summarized in Fig.~\ref{fig:rectangle}.
  \begin{figure}[htbp]
        \centering
        \includegraphics[width=1\linewidth]{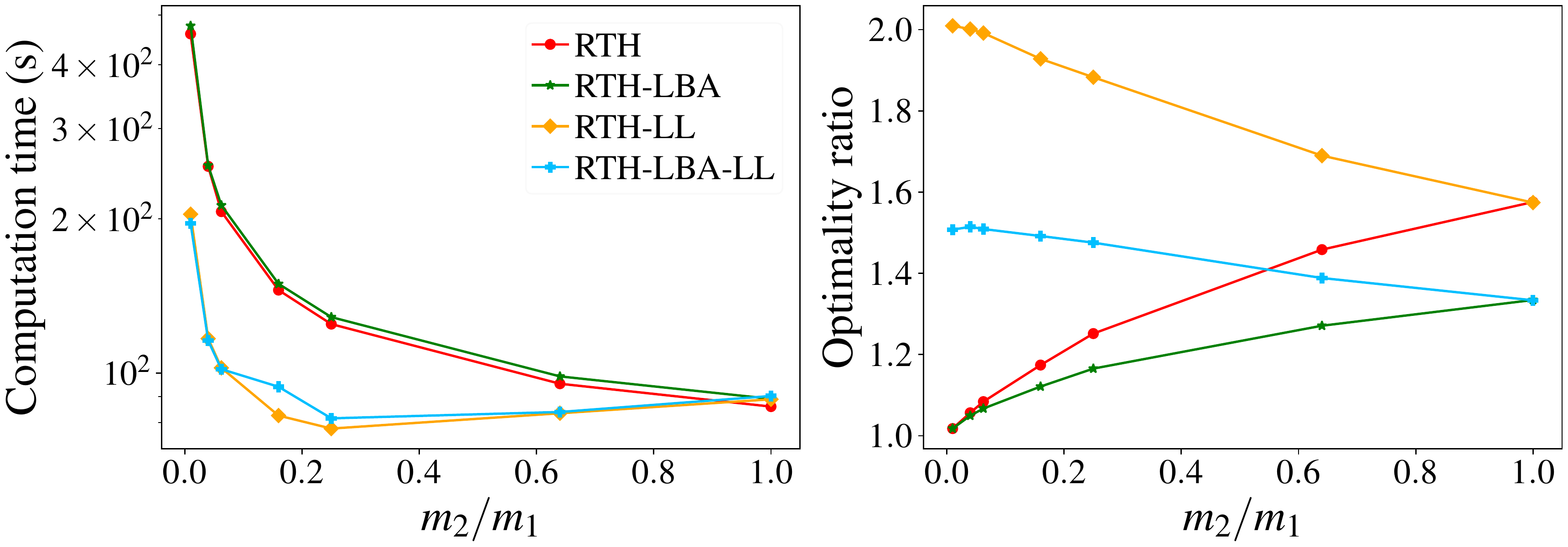}
        \vspace*{-5mm}
        \caption{Computation time and optimality ratios on rectangular grids of varying aspect ratio and $\frac{1}{3}$ robot density.} 
        \label{fig:rectangle} 
        \vspace*{-4mm}
    \end{figure}
    
Interestingly but not surprisingly, the result clearly demonstrates the trade-offs between computation effort and solution optimality. \rth and \rthlba achieve better optimality ratio in comparison to \rth-LL and \rthlba-LL but require more computation time.
Notably, the optimality ratios for \rth and \rthlba are very close to 1 when $m_2:m_1$ is close to 0.
Because the LBA heuristic aims to reduce the possible makespan of the first or third phase of \rta, as $m_1/m_2$ increases, the optimality gap between LBA and non-LBA variants of \rth increases, which clearly demonstrates the advantage of the LBA heuristic. 
\subsection{Special Patterns}
Besides random start and goal settings, we also test \rthlba on many ``special'' instances; two are presented here (Fig.~\ref{fig:new_test}). 
For both settings, $m_1 = m_2$. 
In the first, the ``squares'' setting, robots form concentric square rings and each robot and its goal are centrosymmetric. 
In the second, the ``blocks'' setting, the grid is divided into smaller square blocks (not necessarily $3 \times 3$) containing the same number of robots. robots from one block need to move to another random chosen block.
\rthlba achieves optimality that is fairly close to 1.0 in the square setting and 1.7 in the block setting. The computation time is similar to that of Fig.~\ref{fig: sorting_random}; ECBS does well on optimality but scales poorly (only works on $30\times 30$ grids). For some reason, DDM does very poorly on optimality and is not included. 
    \begin{figure}[htbp]
        \centering
        \includegraphics[width=1\linewidth]{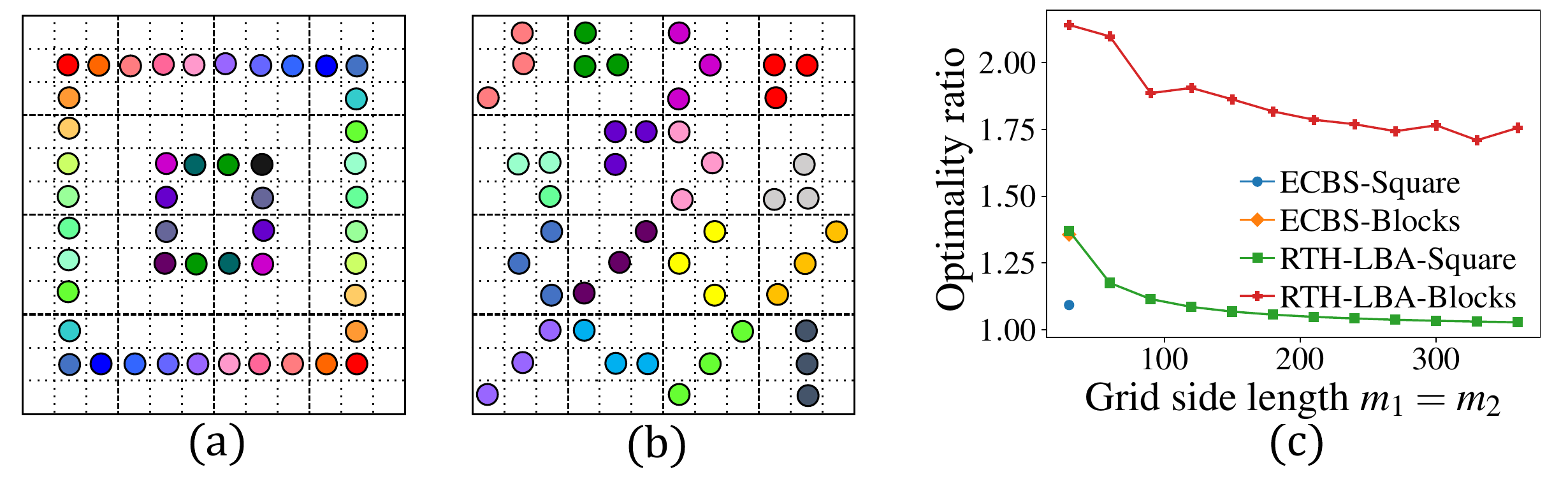}
        \vspace*{-3mm}
        \caption{(a) An illustration of the ``squares'' setting. (b) An illustration of the ``blocks'' setting. (c) Optimality ratios for the two settings for ECBS and \rthlba.} 
        \label{fig:new_test} 
        \vspace*{-3mm}
    \end{figure}

\section{Conclusion and Discussion}\label{sec:conclusion}
In this study, we propose to apply Rubik Tables \cite{szegedy2020rearrangement} to solving \mpp. 
A basic adaptation of \rta, with a more efficient line shuffle routine, enables solving \mpp on grids at maximum robot density, in polynomial time, with previously unachievable optimality guarantee. 
Then, combining \rta, a highway heuristic, and additional matching heuristics, we obtain novel polynomial time algorithms that are provably asymptotically $1 + \frac{m_2}{m_1+m_2}$ makespan-optimal on $m_1\times m_2$ grids with up to $\frac{1}{3}$ robot density, with high probability. 
Similar guarantees are also achieved  with the presence of obstacles and at robot density up to one half. 
In practice, our methods can solve problems on graphs with over $10^5$ number of vertices and $4.5 \times 10^4$ robots to $1.26$ makespan-optimal (which can be better with larger $m_1:m_2$ ratio).
To our knowledge, no previous \mpp solvers provide dual guarantees on low-polynomial running time and practical optimality. 

Our study opens the door for many follow up research directions; we discuss a few here. 

\textbf{New line shuffle routines}. Currently, 
\rtm and \rth only use two/three rows to perform a simulated row shuffle.
Among other restrictions, this requires that the sub-grids used for performing simulated shuffle be well-connected (i.e. obstacle-free or the obstacles are regularly spaced so that there are at least two rows that are not blocked by static obstacles in each motion primitive to simulate the shuffle).
Using more rows or even irregular rows in a simulated row shuffle, it is potentially possible to accommodate larger obstacles and/or support density higher than one half.

\textbf{Better optimality at lower robot density}. 
It is interesting to examine whether further optimality gains can be realized at lower robot density settings, e.g., $\frac{1}{9}$ density or even 
lower, which are still highly practical. We hypothesize that this can be realized by somehow merging the different phases of \rta so that some unnecessary robot travel can be eliminated, after computing an initial plan. 

\textbf{Consideration of more realistic robot models}.
The current study assumes a unit-cost model in which a robot takes a unit amount of time to travel a unit distance and allow turning at every integer time step. 
In practice, robots will need to accelerate/decelerate and also need to make turns. 
Turning can be especially problematic and can cause significant increase in plan execution time, if the original plan is computed using the unit-cost model mentioned above. 
We note that \rth returns solutions where robots move in straight lines most of the time, which is advantageous in comparison to all existing \mpp algorithms, such as ECBS and DDM, which have large number of directional changes in their computed plans. 
it would be interesting to see whether the performance of \rta based \mpp algorithms will further improve as more realistic robot models are adapted. 

\textbf{Life-long \mpp settings}.
Currently, out \rta based \mpp solves are limited to a static setting whereas e-commerce applications of multi-robot motion planning often require solving life-long setting  \cite{Ma2017LifelongMP}. 
%
The metric for evaluating life-long \mpp is often the \emph{throughput}, namely the 
number of goals reached per time step.
We note that \rth also provides optimality guarantees for such settings, e.g., for the setting where $m_1 = m_2 = m$, we have 

\begin{proposition}[Bound for Random Life-Long \mpp]\label{p:life-long}
The direct application of \rth to large-scale life-long \mpp on square grids yields an optimality ratio of $\frac{2}{9}$ on throughput.
\end{proposition}
\ifarxiv
\begin{proof}
We may solve life-long \mpp using \rth in \emph{batches}. 
For each batch with $n$ robots, \rth takes about $3m$ steps; the throughput 
is then $\mathcal{T}_{RTH} = \frac{n}{3m}$.
As for the lower bound estimation of the throughput, the expected Manhattan 
distance in an $m\times m$ square, ignoring inter-robot collisions, is 
$\frac{2m}{3}$.
Therefore, the lower bound throughput for each batch is $\mathcal{T}_{lb}=\frac{3n}{2m}$.
The asymptotic optimality ratio is $\frac{\mathcal{T}_{RTH}}{\mathcal{T}_{lb}}=\frac{2}{9}$.
\end{proof}
\fi

The $\frac{2}{9}$ estimate is fairly conservative because \rth supports much higher 
robot densities not supported by known life-long \mpp solvers. 
Therefore, it appears very promising to develop optimized Rubik Table inspired 
algorithms for solving life-long \mpp problems. 

\section*{Acknowledgments}\label{sec:ack}
This work is supported in part by NSF awards IIS-1845888, CCF-1934924, and IIS-2132972, and an Amazon Research Award. We sincerely thank the anonymous reviewers for their insightful comments and suggestions.


\bibliographystyle{plainnat}
\bibliography{all}

\begin{thebibliography}{56}
\providecommand{\natexlab}[1]{#1}
\providecommand{\url}[1]{\texttt{#1}}
\expandafter\ifx\csname urlstyle\endcsname\relax
  \providecommand{\doi}[1]{doi: #1}\else
  \providecommand{\doi}{doi: \begingroup \urlstyle{rm}\Url}\fi

\bibitem[cov()]{covid-auto}
Warehouse automation market with post-pandemic (covid-19) impact by technology,
  by industry, by geography - forecast to 2026.
\newblock \url{https://www.researchandmarkets.com/r/s6basv}.
\newblock Accessed: 2022-01-05.

\bibitem[Banfi et~al.(2017)Banfi, Basilico, and
  Amigoni]{banfi2017intractability}
Jacopo Banfi, Nicola Basilico, and Francesco Amigoni.
\newblock Intractability of time-optimal multirobot path planning on 2d grid
  graphs with holes.
\newblock \emph{IEEE Robotics and Automation Letters}, 2\penalty0 (4):\penalty0
  1941--1947, 2017.

\bibitem[Barer et~al.(2014)Barer, Sharon, Stern, and
  Felner]{barer2014suboptimal}
Max Barer, Guni Sharon, Roni Stern, and Ariel Felner.
\newblock Suboptimal variants of the conflict-based search algorithm for the
  multi-agent pathfinding problem.
\newblock In \emph{Seventh Annual Symposium on Combinatorial Search}, 2014.

\bibitem[Bitton et~al.(1984)Bitton, DeWitt, Hsaio, and
  Menon]{bitton1984taxonomy}
Dina Bitton, David~J DeWitt, David~K Hsaio, and Jaishankar Menon.
\newblock A taxonomy of parallel sorting.
\newblock \emph{ACM Computing Surveys (CSUR)}, 16\penalty0 (3):\penalty0
  287--318, 1984.

\bibitem[Burkard et~al.(2012)Burkard, Dell'Amico, and
  Martello]{burkard2012assignment}
Rainer Burkard, Mauro Dell'Amico, and Silvano Martello.
\newblock \emph{Assignment problems: revised reprint}.
\newblock SIAM, 2012.

\bibitem[Cheein and Carelli(2013)]{cheein2013agricultural}
Fernando Alfredo~Auat Cheein and Ricardo Carelli.
\newblock Agricultural robotics: Unmanned robotic service units in agricultural
  tasks.
\newblock \emph{IEEE industrial electronics magazine}, 7\penalty0 (3):\penalty0
  48--58, 2013.

\bibitem[Damani et~al.(2021)Damani, Luo, Wenzel, and
  Sartoretti]{damani2021primal}
Mehul Damani, Zhiyao Luo, Emerson Wenzel, and Guillaume Sartoretti.
\newblock Primal $ \_2 $: Pathfinding via reinforcement and imitation
  multi-agent learning-lifelong.
\newblock \emph{IEEE Robotics and Automation Letters}, 6\penalty0 (2):\penalty0
  2666--2673, 2021.

\bibitem[De~Wilde et~al.(2014)De~Wilde, Ter~Mors, and Witteveen]{de2014push}
Boris De~Wilde, Adriaan~W Ter~Mors, and Cees Witteveen.
\newblock Push and rotate: a complete multi-agent pathfinding algorithm.
\newblock \emph{Journal of Artificial Intelligence Research}, 51:\penalty0
  443--492, 2014.

\bibitem[Dekhne et~al.(2019)Dekhne, Hastings, Murnane, and
  Neuhaus]{dekhne2019automation}
Ashutosh Dekhne, Greg Hastings, John Murnane, and Florian Neuhaus.
\newblock Automation in logistics: Big opportunity, bigger uncertainty.
\newblock \emph{McKinsey Q}, pages 1--12, 2019.

\bibitem[Demaine et~al.(2019)Demaine, Fekete, Keldenich, Meijer, and
  Scheffer]{demaine2019coordinated}
Erik~D Demaine, S{\'a}ndor~P Fekete, Phillip Keldenich, Henk Meijer, and
  Christian Scheffer.
\newblock Coordinated motion planning: Reconfiguring a swarm of labeled robots
  with bounded stretch.
\newblock \emph{SIAM Journal on Computing}, 48\penalty0 (6):\penalty0
  1727--1762, 2019.

\bibitem[Erdem et~al.(2013)Erdem, Kisa, Oztok, and
  Sch{\"u}ller]{erdem2013general}
Esra Erdem, Doga~Gizem Kisa, Umut Oztok, and Peter Sch{\"u}ller.
\newblock A general formal framework for pathfinding problems with multiple
  agents.
\newblock In \emph{Twenty-Seventh AAAI Conference on Artificial Intelligence},
  2013.

\bibitem[Erdmann and Lozano-Perez(1987)]{ErdLoz86}
Michael Erdmann and Tomas Lozano-Perez.
\newblock On multiple moving objects.
\newblock \emph{Algorithmica}, 2\penalty0 (1):\penalty0 477--521, 1987.

\bibitem[Ford and Fulkerson(1956)]{ford1956maximal}
Lester~Randolph Ford and Delbert~R Fulkerson.
\newblock Maximal flow through a network.
\newblock \emph{Canadian journal of Mathematics}, 8:\penalty0 399--404, 1956.

\bibitem[Goel et~al.(2013)Goel, Kapralov, and Khanna]{goel2013perfect}
Ashish Goel, Michael Kapralov, and Sanjeev Khanna.
\newblock Perfect matchings in o(n$\backslash$logn) time in regular bipartite
  graphs.
\newblock \emph{SIAM Journal on Computing}, 42\penalty0 (3):\penalty0
  1392--1404, 2013.

\bibitem[Goldenberg et~al.(2014)Goldenberg, Felner, Stern, Sharon, Sturtevant,
  Holte, and Schaeffer]{goldenberg2014enhanced}
Meir Goldenberg, Ariel Felner, Roni Stern, Guni Sharon, Nathan Sturtevant,
  Robert~C Holte, and Jonathan Schaeffer.
\newblock Enhanced partial expansion a.
\newblock \emph{Journal of Artificial Intelligence Research}, 50:\penalty0
  141--187, 2014.

\bibitem[Goldreich(2011)]{goldreich2011finding}
Oded Goldreich.
\newblock Finding the shortest move-sequence in the graph-generalized 15-puzzle
  is np-hard.
\newblock In \emph{Studies in complexity and cryptography. Miscellanea on the
  interplay between randomness and computation}, pages 1--5. Springer, 2011.

\bibitem[Guo and Yu(2022)]{guo2022sub}
Teng Guo and Jingjin Yu.
\newblock Sub-1.5 time-optimal multi-robot path planning on grids in polynomial
  time.
\newblock \emph{arXiv preprint arXiv:2201.08976}, 2022.

\bibitem[Guo et~al.(2021{\natexlab{a}})Guo, Han, and Yu]{9561899}
Teng Guo, Shuai~D. Han, and Jingjin Yu.
\newblock Spatial and temporal splitting heuristics for multi-robot motion
  planning.
\newblock In \emph{2021 IEEE International Conference on Robotics and
  Automation (ICRA)}, pages 8009--8015, 2021{\natexlab{a}}.
\newblock \doi{10.1109/ICRA48506.2021.9561899}.

\bibitem[Guo et~al.(2021{\natexlab{b}})Guo, Han, and Yu]{guo2021spatial}
Teng Guo, Shuai~D Han, and Jingjin Yu.
\newblock Spatial and temporal splitting heuristics for multi-robot motion
  planning.
\newblock In \emph{IEEE International Conference on Robotics and Automation},
  2021{\natexlab{b}}.

\bibitem[{Gurobi Optimization, LLC}(2021)]{gurobi}
{Gurobi Optimization, LLC}.
\newblock {Gurobi Optimizer Reference Manual}, 2021.
\newblock URL \url{https://www.gurobi.com}.

\bibitem[Hall(2009)]{hall2009representatives}
Philip Hall.
\newblock On representatives of subsets.
\newblock In \emph{Classic Papers in Combinatorics}, pages 58--62. Springer,
  2009.

\bibitem[Han and Yu(2020)]{han2020ddm}
Shuai~D Han and Jingjin Yu.
\newblock Ddm: Fast near-optimal multi-robot path planning using
  diversified-path and optimal sub-problem solution database heuristics.
\newblock \emph{IEEE Robotics and Automation Letters}, 5\penalty0 (2):\penalty0
  1350--1357, 2020.

\bibitem[Han et~al.(2018)Han, Rodriguez, and Yu]{han2018sear}
Shuai~D Han, Edgar~J Rodriguez, and Jingjin Yu.
\newblock Sear: A polynomial-time multi-robot path planning algorithm with
  expected constant-factor optimality guarantee.
\newblock In \emph{2018 IEEE/RSJ International Conference on Intelligent Robots
  and Systems (IROS)}, pages 1--9. IEEE, 2018.

\bibitem[H{\"o}nig et~al.(2018)H{\"o}nig, Preiss, Kumar, Sukhatme, and
  Ayanian]{honig2018trajectory}
Wolfgang H{\"o}nig, James~A Preiss, TK~Satish Kumar, Gaurav~S Sukhatme, and
  Nora Ayanian.
\newblock Trajectory planning for quadrotor swarms.
\newblock \emph{IEEE Transactions on Robotics}, 34\penalty0 (4):\penalty0
  856--869, 2018.

\bibitem[Hopcroft et~al.(1984)Hopcroft, Schwartz, and
  Sharir]{hopcroft1984complexity}
John~E Hopcroft, Jacob~Theodore Schwartz, and Micha Sharir.
\newblock On the complexity of motion planning for multiple independent
  objects; pspace-hardness of the" warehouseman's problem".
\newblock \emph{The International Journal of Robotics Research}, 3\penalty0
  (4):\penalty0 76--88, 1984.

\bibitem[Kornhauser et~al.(1984)Kornhauser, Miller, and Spirakis]{KorMilSpi84}
D.~Kornhauser, G.~Miller, and P.~Spirakis.
\newblock Coordinating pebble motion on graphs, the diameter of permutation
  groups, and applications.
\newblock In \emph{Proceedings IEEE Symposium on Foundations of Computer
  Science}, pages 241--250, 1984.

\bibitem[Lam et~al.(2019)Lam, Le~Bodic, Harabor, and Stuckey]{lam2019branch}
Edward Lam, Pierre Le~Bodic, Daniel~Damir Harabor, and Peter~J Stuckey.
\newblock Branch-and-cut-and-price for multi-agent pathfinding.
\newblock In \emph{IJCAI}, pages 1289--1296, 2019.

\bibitem[Leighton and Shor(1989)]{leighton1989tight}
Tom Leighton and Peter Shor.
\newblock Tight bounds for minimax grid matching with applications to the
  average case analysis of algorithms.
\newblock \emph{Combinatorica}, 9\penalty0 (2):\penalty0 161--187, 1989.

\bibitem[Li et~al.(2020)Li, Tinka, Kiesel, Durham, Kumar, and
  Koenig]{li2020lifelong}
Jiaoyang Li, Andrew Tinka, Scott Kiesel, Joseph~W Durham, TK~Satish Kumar, and
  Sven Koenig.
\newblock Lifelong multi-agent path finding in large-scale warehouses.
\newblock In \emph{AAMAS}, pages 1898--1900, 2020.

\bibitem[Li et~al.(2021)Li, Ruml, and Koenig]{li2021eecbs}
Jiaoyang Li, Wheeler Ruml, and Sven Koenig.
\newblock Eecbs: A bounded-suboptimal search for multi-agent path finding.
\newblock In \emph{Proceedings of the AAAI Conference on Artificial
  Intelligence (AAAI)}, 2021.

\bibitem[Luna and Bekris(2011)]{luna2011push}
Ryan~J Luna and Kostas~E Bekris.
\newblock Push and swap: Fast cooperative path-finding with completeness
  guarantees.
\newblock In \emph{Twenty-Second International Joint Conference on Artificial
  Intelligence}, 2011.

\bibitem[Ma and Koenig(2016)]{Ma2016OptimalTA}
Hang Ma and Sven Koenig.
\newblock Optimal target assignment and path finding for teams of agents.
\newblock In \emph{AAMAS}, 2016.

\bibitem[Ma et~al.(2017)Ma, Li, Kumar, and Koenig]{Ma2017LifelongMP}
Hang Ma, Jiaoyang Li, T.~K.~S. Kumar, and Sven Koenig.
\newblock Lifelong multi-agent path finding for online pickup and delivery
  tasks.
\newblock In \emph{AAMAS}, 2017.

\bibitem[Ma et~al.(2019)Ma, Harabor, Stuckey, Li, and Koenig]{ma2019searching}
Hang Ma, Daniel Harabor, Peter~J Stuckey, Jiaoyang Li, and Sven Koenig.
\newblock Searching with consistent prioritization for multi-agent path
  finding.
\newblock In \emph{Proceedings of the AAAI Conference on Artificial
  Intelligence}, volume~33, pages 7643--7650, 2019.

\bibitem[Mason(2019)]{mason2019developing}
Robert Mason.
\newblock Developing a profitable online grocery logistics business: Exploring
  innovations in ordering, fulfilment, and distribution at ocado.
\newblock In \emph{Contemporary Operations and Logistics}, pages 365--383.
  Springer, 2019.

\bibitem[Okumura et~al.(2019)Okumura, Machida, D{\'e}fago, and
  Tamura]{Okumura2019PriorityIW}
Keisuke Okumura, M.~Machida, X.~D{\'e}fago, and Yasumasa Tamura.
\newblock Priority inheritance with backtracking for iterative multi-agent path
  finding.
\newblock In \emph{IJCAI}, 2019.

\bibitem[Perron and Furnon()]{ortools}
Laurent Perron and Vincent Furnon.
\newblock Or-tools.
\newblock URL \url{https://developers.google.com/optimization/}.

\bibitem[Poduri and Sukhatme(2004)]{PodSuk04}
S.~Poduri and G.~S. Sukhatme.
\newblock Constrained coverage for mobile sensor networks.
\newblock In \emph{Proceedings IEEE International Conference on Robotics \&
  Automation}, 2004.

\bibitem[Preiss et~al.(2017)Preiss, H{\"o}nig, Sukhatme, and
  Ayanian]{preiss2017crazyswarm}
James~A Preiss, Wolfgang H{\"o}nig, Gaurav~S Sukhatme, and Nora Ayanian.
\newblock Crazyswarm: A large nano-quadcopter swarm.
\newblock In \emph{IEEE Int. Conf. on Robotics and Automation (ICRA)}, 2017.

\bibitem[Rus et~al.(1995)Rus, Donald, and Jennings]{RusDonJen95}
D.~Rus, B.~Donald, and J.~Jennings.
\newblock Moving furniture with teams of autonomous robots.
\newblock In \emph{Proceedings IEEE/RSJ International Conference on Intelligent
  Robots \& Systems}, pages 235--242, 1995.

\bibitem[Sartoretti et~al.(2019)Sartoretti, Kerr, Shi, Wagner, Kumar, Koenig,
  and Choset]{sartoretti2019primal}
Guillaume Sartoretti, Justin Kerr, Yunfei Shi, Glenn Wagner, TK~Satish Kumar,
  Sven Koenig, and Howie Choset.
\newblock Primal: Pathfinding via reinforcement and imitation multi-agent
  learning.
\newblock \emph{IEEE Robotics and Automation Letters}, 4\penalty0 (3):\penalty0
  2378--2385, 2019.

\bibitem[Sharon et~al.(2013)Sharon, Stern, Goldenberg, and
  Felner]{sharon2013increasing}
Guni Sharon, Roni Stern, Meir Goldenberg, and Ariel Felner.
\newblock The increasing cost tree search for optimal multi-agent pathfinding.
\newblock \emph{Artificial Intelligence}, 195:\penalty0 470--495, 2013.

\bibitem[Sharon et~al.(2015)Sharon, Stern, Felner, and
  Sturtevant]{sharon2015conflict}
Guni Sharon, Roni Stern, Ariel Felner, and Nathan~R Sturtevant.
\newblock Conflict-based search for optimal multi-agent pathfinding.
\newblock \emph{Artificial Intelligence}, 219:\penalty0 40--66, 2015.

\bibitem[Silver(2005)]{silver2005cooperative}
David Silver.
\newblock Cooperative pathfinding.
\newblock \emph{Aiide}, 1:\penalty0 117--122, 2005.

\bibitem[Stern et~al.(2019)Stern, Sturtevant, Felner, Koenig, Ma, Walker, Li,
  Atzmon, Cohen, Kumar, et~al.]{stern2019multi}
Roni Stern, Nathan~R Sturtevant, Ariel Felner, Sven Koenig, Hang Ma, Thayne~T
  Walker, Jiaoyang Li, Dor Atzmon, Liron Cohen, TK~Satish Kumar, et~al.
\newblock Multi-agent pathfinding: Definitions, variants, and benchmarks.
\newblock In \emph{Twelfth Annual Symposium on Combinatorial Search}, 2019.

\bibitem[Surynek(2010)]{surynek2010optimization}
Pavel Surynek.
\newblock An optimization variant of multi-robot path planning is intractable.
\newblock In \emph{Proceedings of the AAAI Conference on Artificial
  Intelligence}, volume~24, 2010.

\bibitem[Surynek(2012)]{surynek2012towards}
Pavel Surynek.
\newblock Towards optimal cooperative path planning in hard setups through
  satisfiability solving.
\newblock In \emph{Pacific Rim International Conference on Artificial
  Intelligence}, pages 564--576. Springer, 2012.

\bibitem[Szegedy and Yu(2020)]{szegedy2020rearrangement}
Mario Szegedy and Jingjin Yu.
\newblock On rearrangement of items stored in stacks.
\newblock In \emph{The 14th International Workshop on the Algorithmic
  Foundations of Robotics}, 2020.

\bibitem[Wagner(2015)]{wagner2015subdimensional}
Glenn Wagner.
\newblock Subdimensional expansion: A framework for computationally tractable
  multirobot path planning.
\newblock 2015.

\bibitem[Wan et~al.(2018)Wan, Gu, Sun, Chen, Huang, and Jia]{wan2018lifelong}
Qian Wan, Chonglin Gu, Sankui Sun, Mengxia Chen, Hejiao Huang, and Xiaohua Jia.
\newblock Lifelong multi-agent path finding in a dynamic environment.
\newblock In \emph{2018 15th International Conference on Control, Automation,
  Robotics and Vision (ICARCV)}, pages 875--882. IEEE, 2018.

\bibitem[Wurman et~al.(2008)Wurman, D'Andrea, and
  Mountz]{wurman2008coordinating}
Peter~R Wurman, Raffaello D'Andrea, and Mick Mountz.
\newblock Coordinating hundreds of cooperative, autonomous vehicles in
  warehouses.
\newblock \emph{AI magazine}, 29\penalty0 (1):\penalty0 9--9, 2008.

\bibitem[Yu(2015)]{yu2015intractability}
Jingjin Yu.
\newblock Intractability of optimal multirobot path planning on planar graphs.
\newblock \emph{IEEE Robotics and Automation Letters}, 1\penalty0 (1):\penalty0
  33--40, 2015.

\bibitem[Yu(2018)]{yu2018constant}
Jingjin Yu.
\newblock Constant factor time optimal multi-robot routing on high-dimensional
  grids.
\newblock \emph{2018 Robotics: Science and Systems}, 2018.

\bibitem[Yu and LaValle(2012)]{YuLav12CDC}
Jingjin Yu and M.~LaValle.
\newblock Distance optimal formation control on graphs with a tight convergence
  time guarantee.
\newblock In \emph{2012 IEEE 51st IEEE Conference on Decision and Control
  (CDC)}, pages 4023--4028, 2012.
\newblock \doi{10.1109/CDC.2012.6426233}.

\bibitem[Yu and LaValle(2013)]{yu2013structure}
Jingjin Yu and Steven~M LaValle.
\newblock Structure and intractability of optimal multi-robot path planning on
  graphs.
\newblock In \emph{Twenty-Seventh AAAI Conference on Artificial Intelligence},
  2013.

\bibitem[Yu and LaValle(2016)]{yu2016optimal}
Jingjin Yu and Steven~M LaValle.
\newblock Optimal multirobot path planning on graphs: Complete algorithms and
  effective heuristics.
\newblock \emph{IEEE Transactions on Robotics}, 32\penalty0 (5):\penalty0
  1163--1177, 2016.

\end{thebibliography}

\end{document}